%% file: main.tex
\documentclass[twocolumn,journal,twoside]{IEEEtran}

\usepackage[T1]{fontenc}
\usepackage[bookmarks=false]{hyperref}
\usepackage[cmex10]{amsmath}
\usepackage{pgfplots}
\makeatletter
\def\BState{\State\hskip-\ALG@thistlm}
\makeatother

\usepackage{cite}
\usepackage[utf8]{inputenc}

\usepackage{graphicx}
\usepackage{subcaption}
\usepackage{pstricks}
\usepackage{color}
\usepackage{bbm}
\usepackage{tcolorbox}
\usepackage{tikz}
\usepackage{pgfplots}
\usepackage{wrapfig}
\usepackage{lipsum}  
\usetikzlibrary{positioning}
\usepackage{tcolorbox}
\usepackage{amssymb}
\usepackage{amsmath,amsfonts,amssymb,amsthm}

\usepackage{mathtools}
\usepackage{commath}
\usepackage{relsize}
\usepackage{bbm}
\usepackage{bm}
\usepackage[font={small}]{caption}
\usepackage{comment,color,soul}
\usepackage{enumerate} 
\usetikzlibrary{arrows,automata}
\usepackage{amsfonts}

\usepackage{mathtools}
\usepackage{url}
\usepackage{algorithm}
\usepackage{algpseudocode}
\usepackage{lipsum}
\usepackage[thinlines]{easytable}

\makeatletter
\def\thm@space@setup{\thm@preskip=2pt
        \thm@postskip=2pt \itshape}
\makeatother

\newtheoremstyle{newstyle}
{} 
{} 
{\mdseries} 
{} 
{\bfseries} 
{.} 
{ } 
{} 

\theoremstyle{newstyle}

\newtheorem{theorem}{Theorem}
\newtheorem{lemma}{Lemma}

\theoremstyle{definition}
\newtheorem{example}{Example}

\newtheorem{remark}{Remark}
\newtheorem{claim}{Claim}

\IEEEoverridecommandlockouts
\newcommand{\Expc}{\mathbb{E}}
\newcommand{\Prob}{\mathbb{P}}
\newcommand{\bA}{\mathbf{A}}
\newcommand{\ba}{\mathbf{a}}
\newcommand{\bbb}{\mathbf{b}}
\newcommand{\bbm}{\mathbf{m}}
\newcommand{\bB}{\mathbf{B}}
\newcommand{\bG}{\mathbf{G}}
\newcommand{\bg}{\mathbf{g}}
\newcommand{\bx}{\mathbf{x}}

\newcommand{\cD}{\mathcal{D}}
\newcommand{\cO}{\mathcal{O}}
\newcommand{\vv}{\mathbf{v}}

\newcommand{\saurav}[1]{{\color{blue}#1}}

 \definecolor{constructCluster}{HTML}{2B83BA}

\usepackage{mathtools}
\newcommand{\argmin}{\operatornamewithlimits{arg\,min}}
\usepackage{tcolorbox}
\usepackage{amsmath,amssymb}
\usepackage{mathtools}
\usepackage{enumitem}
\usepackage{booktabs}

\begin{document}
\sloppy
\title{CodedReduce: A Fast and Robust Framework for Gradient Aggregation in Distributed Learning}
\author{Amirhossein Reisizadeh$^{*}$, Saurav Prakash$^{*}$, Ramtin Pedarsani, Amir Salman Avestimehr
\thanks{$^*$Authors have equal contribution.}
\thanks{Manuscript received August 25, 2020; revised April 11, 2021 and August 3, 2021; accepted August 8, 2021; approved by IEEE/ACM Transactions on Networking Editor R. La. }
\thanks{A. Reisizadeh and R. Pedarsani are with the Department
of Electrical and Computer Engineering, University of California, Santa Barbara, Santa Barbara, CA 93106 USA (e-mail: reisizadeh@ucsb.edu; ramtin@ece.ucsb.edu).}
\thanks{S. Prakash and A. S. Avestimehr are with the the Department
of Electrical and Computer Engineering, University of Southern California, Los Angeles, CA 90089 USA
(e-mail: sauravpr@usc.edu; avestimehr@ee.usc.edu).}
\thanks{We sincerely thank the editor and all the reviewers for their  valuable feedback and  detailed  comments. This work is supported by NSF grants CNS-2003035, CCF-1408639, CCF-1755808, NETS-1419632, Office of Naval Research (ONR) award N000141612189, NSA grant, a research gift from Intel and by Defense Advanced Research Projects Agency (DARPA) under Contract No. HR001117C0053. The views, opinions, and/or findings expressed are those of the author(s) and should not be interpreted as representing the official views or policies of the Department of Defense or the U.S. Government.}
\thanks{A part of this work was presented in IEEE International Symposium on
Information Theory, 2019 \cite{reisizadeh2019tree}.}
\thanks{1558-2566 © 2021 IEEE. Personal use is permitted, but republication/redistribution requires IEEE permission. For more information, see https://www.ieee.org/publications/rights/index.html }
}
\maketitle

\begin{abstract}
\input{0-abstract}
\end{abstract}

\section{Introduction}
\input{1-intro.tex}

\section{Problem Setup and Background}\label{sec:background}
\input{2-background.tex}

\section{Proposed CodedReduce Scheme}\label{sec:codedreduce}
\input{3-CR.tex}

\section{Empirical Evaluation of \textsf{CR}}\label{sec:experiments}
\input{6-experiments.tex}

\section{Conclusion}\label{sec:conclusion}
\input{7-conclusion.tex}


\bibliographystyle{ieeetr}
\bibliography{biblio}

\input{8-bios.tex}

\newpage 

\begin{appendices}
\onecolumn
\section{Pseudo-code for Computation Allocation Sub-routine}
\label{appA}
\input{5-appA.tex}
\section{Pseudo-code for CodedReduce Scheme}
\label{appB}
\input{5-appB.tex}
\twocolumn
\section{Proof of Theorem \ref{thm:CRoptimality}}
\label{appC}
\input{5-appC.tex}
\section{Proof of Theorem \ref{thm:CRtime}}
\label{appD}
\input{5-appD.tex}
\end{appendices}

\newpage


\end{document}

%% file: 0-abstract.tex
%
We focus on the commonly used synchronous Gradient Descent paradigm for large-scale distributed learning, for which there has been a growing interest to develop efficient and robust gradient aggregation strategies that overcome two key system bottlenecks: communication bandwidth and stragglers' delays. 
%
In particular, Ring-AllReduce (\textsf{RAR}) design has been proposed to avoid bandwidth bottleneck at any particular node by allowing each worker to only communicate with its neighbors that are arranged in a logical ring. 
On the other hand, Gradient Coding (\textsf{GC}) has been recently proposed to mitigate stragglers in a master-worker topology by allowing carefully designed redundant allocation of the data set to the workers. 
We propose a joint communication topology design and data set allocation strategy, named CodedReduce (\textsf{CR}), that combines the best of both \textsf{RAR} and \textsf{GC}. That is, it parallelizes the communications over a tree topology leading to efficient bandwidth utilization, and carefully designs a redundant data set allocation and coding strategy at the nodes to make the proposed gradient aggregation scheme robust to stragglers.
In particular, we quantify the communication parallelization gain and resiliency of the proposed \textsf{CR} scheme, and prove its optimality when the communication topology is a regular tree. 
%
%
Moreover, we characterize the expected run-time of \textsf{CR} and show order-wise speedups compared to the benchmark schemes.
%
%
Finally, we empirically evaluate the performance of our proposed \textsf{CR} design over Amazon EC2 and demonstrate that it achieves speedups of up to $27.2\times$ and $7.0\times$, respectively over the benchmarks \textsf{GC} and \textsf{RAR}.

%% file: 1-intro.tex
Modern machine learning algorithms are now used in a wide variety of domains. However, training a large-scale model  over a massive data set is an extremely computation and storage intensive task, e.g. training ResNet with more than 150 layers and hundreds of millions of parameters over the data set ImageNet with more than 14 million images. 
As a result, there has been significant interest in developing distributed learning strategies that speed up the training of learning models (e.g., \cite{dekel2012optimal,zinkevich2010parallelized,chen2016revisiting,recht2011hogwild,dean2012large,chilimbi2014project,abadi2016tensorflow}). 

In the commonly used Gradient Descent (GD) paradigm for learning, parallelization can be achieved by arranging the machines in a master-worker setup. Through a series of iterations, the master is responsible for updating the underlying model from the results received from the workers, where they compute the partial gradients using their local data batches and upload to the master at each iteration. For the master-worker setup, both synchronous  and asynchronous methods have been developed \cite{dekel2012optimal,zinkevich2010parallelized,chen2016revisiting,recht2011hogwild,dean2012large,chilimbi2014project}. In synchronous settings, all the workers wait for each other to complete the gradient computations, while in asynchronous methods, the workers continue the training process after their local gradient is computed. While synchronous approaches provide better generalization behaviors than the asynchronous ones \cite{cong2017efficient,chen2016revisiting}, they face major system bottlenecks due to (1) bandwidth congestion at the master due to concurrent communications from the workers to the master  \cite{patarasuk2009bandwidth}; and (2) the delays caused by slow workers or stragglers that significantly increase the run-time \cite{recht2011hogwild}. 

\begin{figure}[h!] 
    \centering
    \includegraphics[width=.43\textwidth]{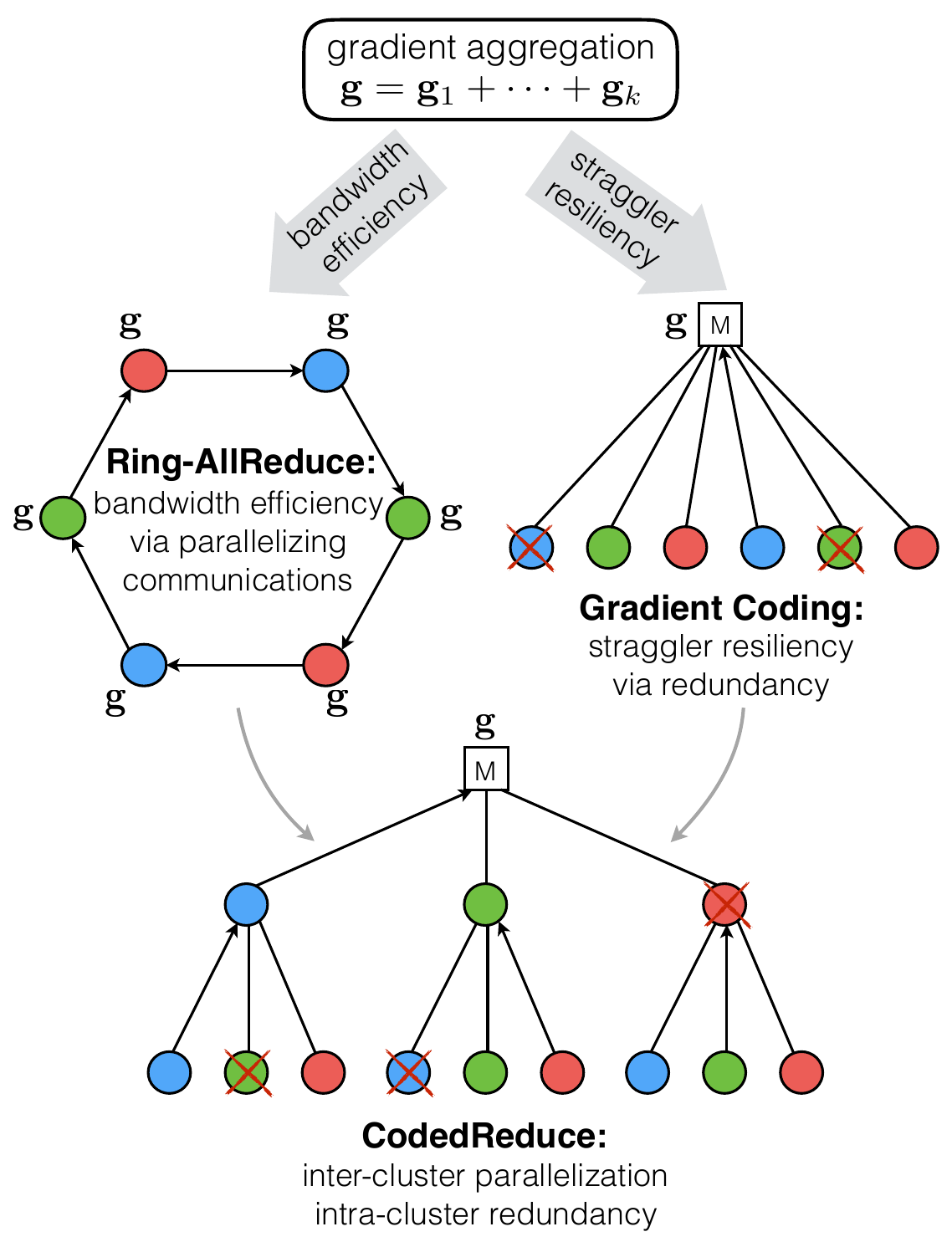}
    \caption{ Illustration of \textsf{RAR}, \textsf{GC} and \textsf{CR}:
    In \textsf{RAR}, workers communicate  only with their neighbors on a ring, which results in high bandwidth utilization; however, \textsf{RAR} is prone to stragglers. \textsf{GC} is robust to stragglers by doing redundant computations at workers; however, \textsf{GC} imposes bandwidth bottleneck at the master. \textsf{CR} achieves the benefits of both worlds, providing high bandwidth efficiency along with straggler resiliency.} 
    \label{fig:mainfig}
\end{figure}

To alleviate the communication bottleneck in distributed learning, various bandwidth efficient strategies have been proposed \cite{patarasuk2007bandwidth, thakur2005optimization,kandalla2010designing}. Particularly, Ring-AllReduce (\textsf{RAR})  \cite{patarasuk2009bandwidth} strategy has been proposed  by allowing each worker to only communicate with its neighbors that are arranged in a logical ring. More precisely, the data set, $\cD$, is uniformly distributed among $N$ workers and each node combines and passes its partial gradient along the ring such that at the end of the collective operation, each worker has a copy of the full gradient $\bg$ (Figure \ref{fig:mainfig}). Due to the master-less topology of  \textsf{RAR}, it avoids bandwidth bottleneck at any particular node.  Furthermore, as shown in \cite{patarasuk2007bandwidth}, \textsf{RAR} is provably bandwidth optimal and induces $\cO(1)$ communication overhead that does not depend on the number of distributed workers. 
As a result, \textsf{RAR} has recently become a central component in distributed deep learning for model updating \cite{gibiansky2017bringing,sergeev2018horovod,jin2016scale}. More recent approaches to mitigate bandwidth bottleneck in distributed gradient aggregation include compression and quantization of the gradients \cite{li2018pipe,yu2018gradiveq,sun2020lazily}.

Despite being bandwidth efficient, AllReduce-type algorithms are inherently sensitive to stragglers, which makes them prone to significant performance degradation and even complete failure if \textit{any} of the workers slows down. Straggler bottleneck becomes even more significant as the cluster size increases \cite{dean2013tail,zhao2017efficient}.


One approach to mitigate stragglers in distributed computation is to introduce computational redundancy via replication. \cite{zaharia2008improving} proposes to replicate the straggling task on other available nodes. In \cite{pu2015low}, the authors propose a partial data replication for robustness. Other relevant replication based strategies have been proposed in \cite{ananthanarayanan2013effective,wang2014efficient,shah2016redundant}. Recently, coding theoretic approaches have also been proposed for straggler mitigation \cite{lee2016speeding,dutta2016short,reisizadeh2017latency,yu2017polynomial,karakus2017straggler,ye2018communication,yu2018lagrange,narra2019slack,narra2020collage,yang2019timely,dhakal2019coded}. 
Specifically, Gradient Coding (\textsf{GC}) \cite{tandon2017gradient} has been proposed to alleviate stragglers in distributed gradient aggregation in a master-worker topology (Figure \ref{fig:mainfig}). In \textsf{GC}, the data set $\cD$ is carefully and redundantly distributed among the $N$ workers where each worker computes a \emph{coded} gradient from its local batch. The master node waits for the results of \emph{any} $N-S$ workers and recovers the total gradient $\bg$, where the design parameter $S$ denotes the maximum number of stragglers that can be tolerated. Therefore, \textsf{GC} prevents the master from waiting for \emph{all} the workers to finish their computations, and it was shown to achieve significant speedups over the classical uncoded master-worker setup \cite{tandon2017gradient}.


However, as the cluser size gets large, \textsf{GC} suffers from significant network congestion at the master. In particular, the communication overhead increases to  $\cO(N)$, as the master needs to receive messages from $\cO(N)$ workers.
Thus, it is essential to design distributed learning strategies that alleviate stragglers while imposing low communication overhead across the cluster. Consequently, our goal in this paper is to answer the following fundamental question:

\begin{tcolorbox}
\begin{center}
\emph{Can we achieve the communication parallelization of \textsf{RAR} and the straggler toleration of \textsf{GC} simultaneously in distributed gradient aggregation?}
\end{center}
\end{tcolorbox}
\noindent


We answer this question in the affirmative. As the main contribution of this paper, we propose a joint design of data allocation and communication strategy that is robust to stragglers, alongside being bandwidth efficient. Specifically, we propose a scalable and robust scheme for synchronous distributed gradient aggregation, called CodedReduce (\textsf{CR}).


There are two key ideas behind \textsf{CR}. Firstly, we use a logical tree topology for communication  consisting of a master node, $L$ layers of workers, where each \emph{parent} node has $n$ \emph{children} nodes (Figure \ref{fig:mainfig}). In the proposed configuration, each node communicates only with its parent node for \emph{downloading} the updated model and \emph{uploading} partial gradients. As in the classical master-worker setup, the root node (master) recovers the full gradient and updates the model. Except for the leaf nodes, each node receives \emph{enough} number of \textit{coded} partial gradients from its children, combines them with its local and partial gradient and uploads the result to its parent. This distributed communication strategy alleviates the communication bottleneck at the nodes, as  multiple parents can concurrently receive from their children. Secondly, the coding strategy utilized in \textsf{CR} provides robustness to stragglers. Towards this end, we exploit ideas from \textsf{GC} and propose a data allocation and communication strategy such that \emph{each} node needs to only wait for 
\emph{any} $n-s$ of its children to return their results. 

The theoretical guarantees of the proposed \textsf{CR} scheme are two-fold. First, we characterize the computation load introduced by the proposed \textsf{CR} and prove that for a fixed straggler resiliency, \textsf{CR} achieves the 
optimal \emph{computation load} (relative size of the assigned local data set to the total data set) among all the robust gradient aggregation schemes over a fixed tree topology. Moreover, \textsf{CR} significantly improves upon \textsf{GC} in the computation load  of the workers. More precisely,  to be robust to straggling/failure of $\alpha$ fraction of the children, \textsf{GC} loads each worker with $\approx \alpha$ fraction of the total data set, while \textsf{CR} assigns only $\approx \alpha^L$ fraction of the total data set, which is a major improvement. Secondly, we model the workers' computation times as shifted exponential random variables and asymptotically characterize the average latency of \textsf{CR}, that is the expected time to aggregate the gradient at the master node as the number of workers tends to infinity. This analysis further demonstrates how \textsf{CR} alleviates the bandwidth efficiency and speeds up the training process by parallelizing the communications via a tree. 


In addition to provable theoretical guarantees, the proposed \textsf{CR} scheme offers substantial improvements in practice. As a representative case, Figure \ref{fig:stats} provides the gradient aggregation time averaged over  many gradient descent iterations implemented over Amazon EC2 clusters. Compared to three benchmarks -- classical Uncoded Master-Worker (\textsf{UMW}), \textsf{GC}, \textsf{RAR}  -- the proposed \textsf{CR} scheme attains speedups of $22.5 \times$, $6.4 \times$ and  $4.3 \times$, respectively.

\begin{figure}[h!] 
    \centering
    \includegraphics[width=.47\textwidth]{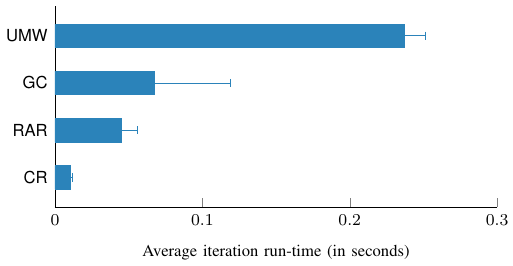}
    \caption{ Average iteration time for gradient aggregation in different schemes  \textsf{CR}, \textsf{RAR}, \textsf{GC} and \textsf{UMW}: 
Training a linear model is implemented on a cluster of $N=84$ \texttt{t2.micro} instances. 
} \label{fig:stats}
\end{figure}

%% file: 2-background.tex
In this section, we provide the problem setup followed by a brief background on \textsf{RAR} and \textsf{GC} and their corresponding straggler resiliency and communication parallelization.

\subsection{Problem Setting}
Many machine learning tasks involve fitting a model over a training data set by minimizing a loss function. For a given labeled data set $\cD = \{ \bx_j \in \mathbb{R}^{p+1}  : j=1,\cdots, d \}$, the goal is to solve the following optimization problem:
\begin{equation}\label{eq:minb}
    \theta^* = \argmin_{\theta \in \mathbb{R}^{p}} \sum_{\bx \in \cD} \ell \left( \theta; \bx \right) + \lambda R(\theta),
\end{equation}
where $\ell(\cdot)$ and $R(\cdot)$ respectively denote the loss and regularization functions, and the optimization problem is parameterized by $\lambda$. One of the most popular ways of solving \eqref{eq:minb} in distributed learning is to use the Gradient Descent (GD) algorithm. More specifically, under standard convexity assumptions, the following sequence of model updates $\{ \theta^{(t)} \}_{t=0}^{\infty}$ converges to the optimal solution $\theta^{*}$:
\begin{equation}\label{eq:gd}
    \theta^{(t+1)} = h_R \left( \theta^{(t)}, \bg \right),
\end{equation}
where $h_R(\cdot)$ is a gradient-based optimizer depending on the regularizer $R(\cdot)$ and 
\begin{equation}\label{eq:g}
    \bg = \sum_{\bx \in \cD} \nabla \ell \left(\theta^{(t)}; \bx \right),
\end{equation}
denotes the gradient of the loss function evaluated at the model at iteration $t$ over the data set ${\cD}$. Under certain assumptions, the iterations in (\ref{eq:gd}) converge to a local optimum in the non-convex case, as well. For instance, if all the saddle points of a smooth non-convex objective are strict-saddle, then the iterations in \eqref{eq:gd} converge to a local minimum \cite{lee2016gradient}. The core component of the iterations defined in (\ref{eq:gd}) is the computation of the gradient vector $\bg$ at each iteration. At scale, due to limited storage and computation capacity of the computing nodes, gradient aggregation task (\ref{eq:g}) has to be carried out over distributed nodes. This parallelization, as we discussed earlier, introduces two major bottlenecks: stragglers and bandwidth contention. 
The goal of the distributed gradient aggregation scheme is to provide straggler resiliency as well as communication parallelization. At a high level, straggler resiliency, $\alpha$, refers to the fraction of the straggling workers that the distributed aggregation scheme is robust to, and communication parallelization gain, $\beta$, quantifies the number of simultaneous communications in the network by distributed nodes compared to only one simultaneous communication in a single-node (master-worker) aggregation scheme.

Next, we discuss the data allocation and communication strategy of two synchronous gradient aggregation schemes in distributed learning and their corresponding straggler resiliency and communication parallelization gain.

\subsection{Ring-AllReduce}

In AllReduce-type aggregation schemes, the data set is uniformly distributed over $N$ worker nodes $\{W_1,\cdots,W_N\}$ which coordinate among themselves in a master-less setting to aggregate their partial gradients and compute the aggregate gradient $\bg$ at each worker. 
Particularly in \textsf{RAR}, each worker $W_i$ partitions its local partial gradient into $N$ segments  $\vv_{1,i},\cdots, \vv_{N,i}$. In the first round, $W_i$ transmits $\vv_{i,i}$ to $W_{i+1}$. Each worker then adds up the received segment to the corresponding segment of its local gradient, i.e., $W_i$ obtains $\vv_{i-1,i-1}+ \vv_{i-1,i}$. In the second round, the reduced segment is forwarded to the neighbor and added up to the corresponding segment. Proceeding similarly, at the end of $N-1$ rounds, each worker has a unique segment of the full gradient, i.e., $W_i$ has $\vv_{i+1,1}+\ldots+ \vv_{i+1,N}$. After the reduce-scatter phase, the workers execute the collective operation of AllGather where the full gradient $\bg$ becomes available at each node.  The \textsf{RAR} operation for a cluster of three workers is illustrated in Figure \ref{fig:Ring-AllReduce}. 

It is clear that \textsf{RAR} cannot tolerate \emph{any} straggling nodes since the communications are carried out over a ring and each node requires its neighbor's result to proceed in the ring, i.e., the straggler resiliency for \textsf{RAR} is $\alpha_{\textsf{RAR}}=0$. However,  the ring communication design in \textsf{RAR} alleviates the communication congestion at busy nodes, and achieves communication parallelization gain $\beta_{\textsf{RAR}}= \Theta (N)$ which is optimal \cite{patarasuk2009bandwidth}. 


\begin{figure}[h!] 
    \centering
    \includegraphics[width=0.5\textwidth]{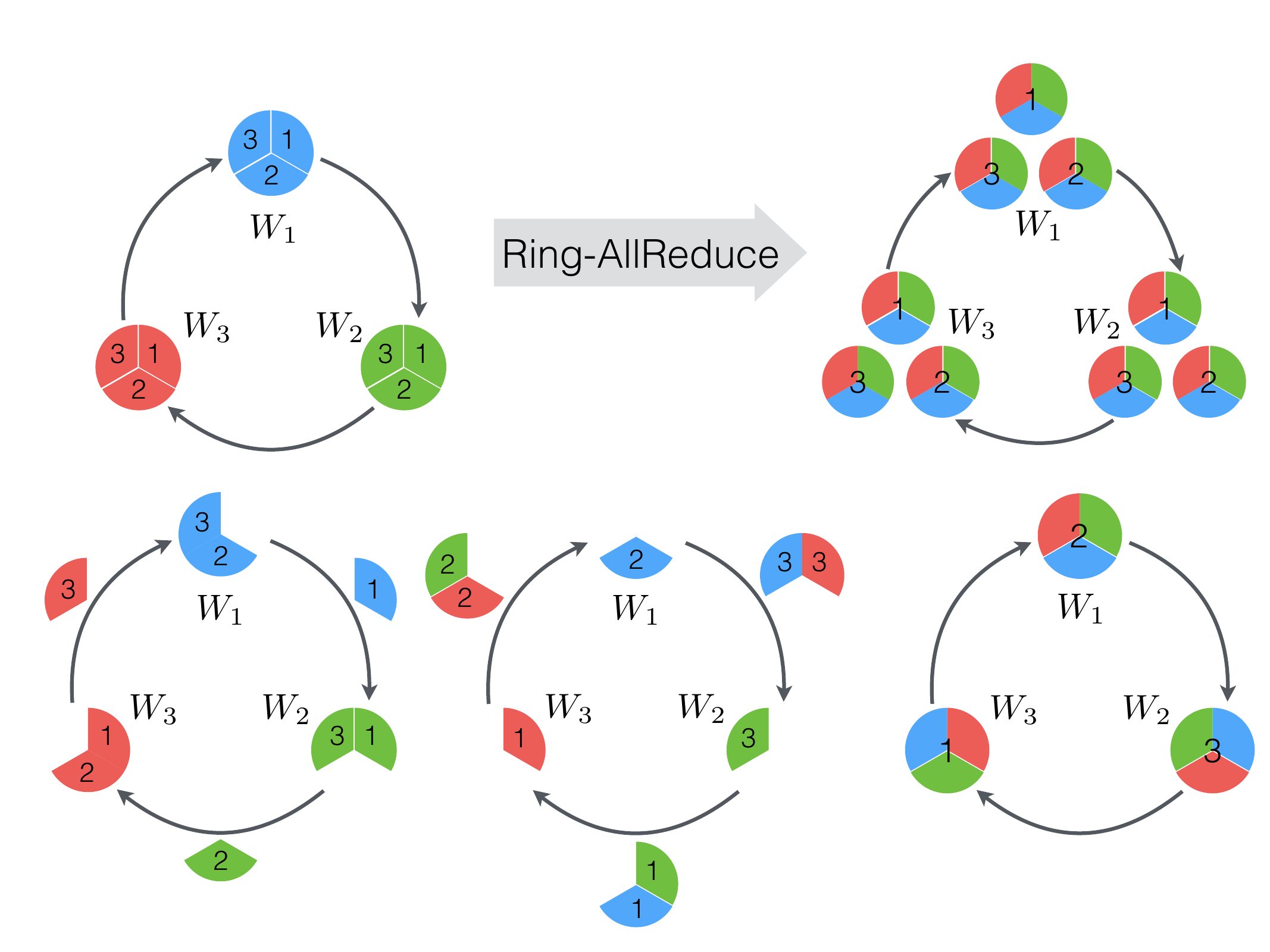}
    \caption{Illustration of communication strategy in \textsf{RAR} for $N=3$ workers.} \label{fig:Ring-AllReduce1}
    \label{fig:Ring-AllReduce}
\end{figure}

\subsection{Gradient Coding}
\label{subsec:GC}
Gradient Coding (\textsf{GC}) \cite{tandon2017gradient} was recently proposed to provide straggler resiliency in a master-worker topology with one master node and $N$ distributed worker nodes $\{W_1,\cdots,W_N\}$ as depicted in Figure \ref{fig:mainfig}. We start the description of 
\textsf{GC} with an illustrative example. 
\begin{example}[Gradient Coding]
To make gradient aggregation over $N=3$ workers robust to any $S=1$ straggler, \textsf{GC} partitions the data set to $\{ \cD_1, \cD_2, \cD_3 \}$ and assigns $2$ partitions to each worker as depicted in Figure \ref{fig:GC-example}. Full gradient $\bg=\bg_1+\bg_2+\bg_3$ can be recovered from any $N-S=2$ workers, e.g., the master recovers $\bg$ from  $W_1$ and $W_2$ by combining their results as $\bg = 2 \left( \frac{1}{2}\bg_1+\bg_2 \right) - \left( \bg_2 - \bg_3 \right) $.

\begin{figure}[h!] 
    \centering
    \includegraphics[width= 0.24\textwidth]{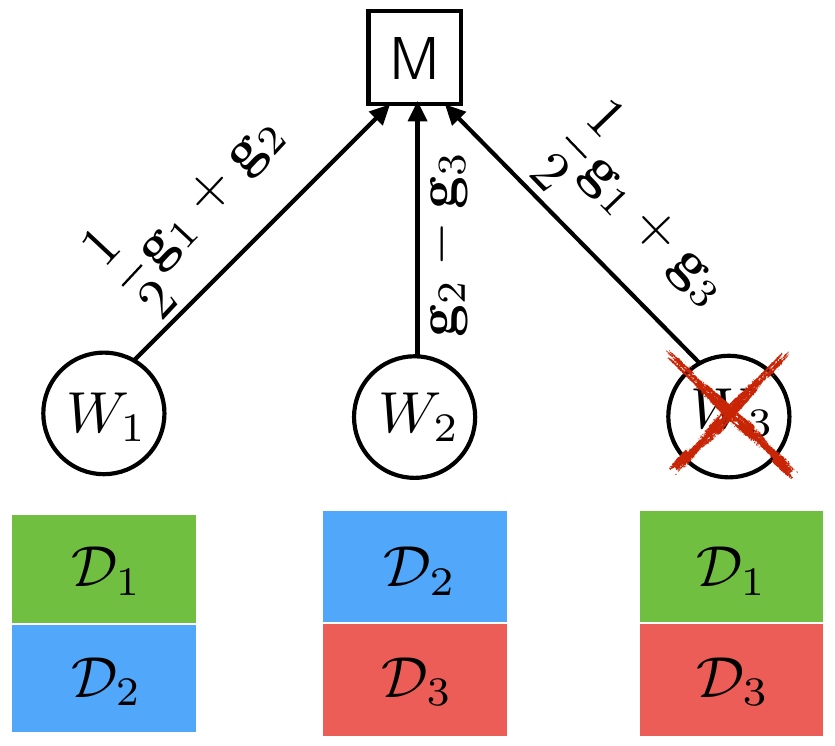}
    \caption{Illustration of data allocation and communication strategy in \textsf{GC} for $N=3$ workers.} \label{fig:GC-example}
\end{figure}

\end{example}

In general, to be robust to \emph{any} $S \in [N] = \{1,\cdots,N\}$ stragglers, \textsf{GC} uniformly partitions the data set $\cD$ to $\{ \cD_1,\cdots,\cD_k \}$ (e.g. $k=N$) with corresponding partial gradients $\bg_1, \cdots, \bg_k$ and distributes them redundantly among the workers such that each partition is placed in $S+1$ workers, thus achieving a computation load of $r_{\textsf{GC}}=\frac{S+1}{N}$. 
Let matrix $\bG = [\bg_1,\cdots,\bg_k]^{\top} \in \mathbb{R}^{k \times p}$ denote the collection of partial gradients. Each worker $W_i$ then computes its local partial gradients 
and sends $\bbb_i \bG$ to the master, where  $\bB = [\bbb_1;\cdots;\bbb_N] \in \mathbb{R}^{N \times k}$ denotes the encoding matrix, i.e. non-zero elements in $\bbb_i$ specifies the partitions stored in worker $W_i$.  Upon receiving the results of any $N-S$ workers, the master recovers the total gradient $\bg$ by linearly combining the received results, that is $\bg = \ba_f \bB \bG$
where the row vector $\ba_f \in \mathbb{R}^{1 \times N}$ corresponds to a particular set of $S$ stragglers and $\bA = [\ba_1;\cdots;\ba_F]$ denotes the decoding matrix with $F={N \choose S}$ distinct straggling scenarios. The \textsf{GC} algorithm designs encoding and decoding matrices $(\bB,\bA)$ such that, in the worst case, the full gradient $\bg$ is recoverable from the results of \emph{any} $N-S$ out of $N$ workers, i.e. straggler resiliency $\alpha_{\textsf{GC}}=S/N$ is attained. Although \textsf{GC} prevents the master to wait for \emph{all} the workers to finish their computations, it requires simultaneous communications from the workers that will cause congestion at the master node, and lead to  parallelization gain $\beta_{\textsf{GC}}= \Theta (1)$ for a constant resiliency.

Having reviewed \textsf{RAR} and \textsf{GC} strategies and their resiliency and parallelization properties, we now informally provide the guarantees of our proposed \textsf{CR} scheme in the following remark.
\begin{remark}
\textsf{CR} arranges the available $N$ workers via a tree configuration with $L$ layers of nodes and each parent having $n$ children, i.e. $N=n+\cdots+n^L$. The proposed data allocation and communication strategy in \textsf{CR} results in communication parallelization gain $\beta_{\textsf{CR}} = \Theta(N^{1-1/L})$ which approaches $\beta_{\textsf{RAR}} = \Theta(N)$ for large $L$. Moreover, given a computation load $0\leq r \leq 1$, \textsf{CR} is robust to straggling of  $\alpha_{\textsf{CR}} \approx r^{1/L}$ fraction of the children per any parent in the tree, while \textsf{GC} is robust to only  $\alpha_{\textsf{GC}} \approx r$ fraction of nodes and \textsf{RAR} has no straggler resiliency. 
Therefore, \textsf{CR} achieves the best of \textsf{RAR} and \textsf{GC}, simultaneously. Table \ref{tab:schemes} summarizes these results and Theorems \ref{thm:CRoptimality} and \ref{thm:CRtime}  formally characterize such guarantees.

\end{remark}

\begin{table}[t]
\caption{Communication parallelization gain and straggler resiliency of three designs \textsf{RAR}, \textsf{GC}, and \textsf{CR} in a system with $N$ nodes with computation load $r$, where \textsf{CR} has a tree communication topology of $L$ layers.}
\label{tab:schemes}
\begin{center}
\begin{small}
\begin{sc}
\begin{tabular}{lcc}
\toprule
 Scheme & \begin{tabular}{@{}c@{}@{}}Straggler   \\ Resiliency \\ ($\alpha$)\end{tabular} & \begin{tabular}{@{}c@{}@{}}Communication  \\ Parallelization Gain \\ ($\beta$)\end{tabular}  \\
\midrule
\textsf{RAR}    & $0$ & $\Theta(N)$ \\
\textsf{GC}     & $r$ & $\Theta(1)$\\
\textsf{CR}     & $r^{1/L}$ & $\Theta \left( N^{1-1/L} \right)$ \\
\bottomrule
\end{tabular}
\end{sc}
\end{small}
\end{center}
\vskip -0.1in
\end{table}

%% file: 3-CR.tex
In this section, we first present our proposed CodedReduce (\textsf{CR}) scheme by describing data set allocation and communication strategy at the nodes followed by an illustrative example. Then, we provide theoretical guarantees of \textsf{CR} and conclude the section with optimality of \textsf{CR}.

\subsection{Description of \textsf{CR} Scheme } \label{subsec:CRdescr}

 Let us start with the proposed network configuration. \textsf{CR} arranges the communication pattern among the nodes via a \emph{regular} tree structure as defined below.
An $(n,L)$--regular tree graph $T$ consists of a master node and $L$ layers of worker nodes. At any layer (except for the lowest), each \emph{parent} node is connected to $n$ \emph{children} nodes in the lower layer, i.e. there is a total of $N=n+\cdots+n^L$ nodes (See Figure \ref{fig:n-tree}). Each node of the tree is identified with a pair $(l,i)$, where $l \in [L]$ and $i \in [n^l]$ denote the corresponding layer and the node's index in that layer, respectively. Furthermore, $T(l,i)$ denotes the sub-tree with the root node $(l,i)$.
\begin{figure}[h!] 
    \centering
    \includegraphics[width=.35\textwidth]{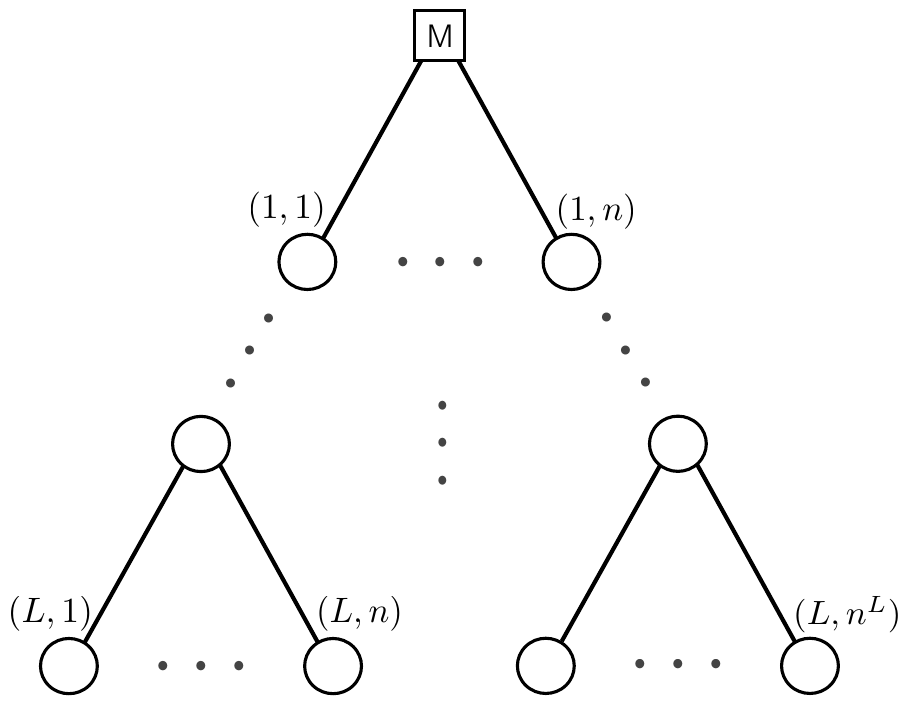}
    \caption{$(n,L)$--regular tree topology.} \label{fig:n-tree}
\end{figure}

We next introduce a notation that eases the algorithm description.
We associate a real scalar $b$ to all the data points in a generic data set $\cD$, denoting it by $b \cD$, and define the gradient over $b \cD$ as $\bg_{b \cD} = b \bg_{\cD} = b \sum_{\bx \in \cD} \nabla \ell (\theta^{(t)}; \bx ) $. As a building block of \textsf{CR}, we define the sub-routine \textsc{CompAlloc} in which given a generic data set $\cD$, $n$ workers are carefully assigned with data partitions and combining coefficients such that the full gradient over $\cD$ is retrievable from the computation results of any $n-s$ workers (Pseudo-code  in Appendix \ref{appA}).

\noindent
\textbf{\textsc{CompAlloc}}: For specified $n$ and $s$, \textsf{GC} (Algorithm 2 in \cite{tandon2017gradient}) constructs the encoding matrix $\bB = [\bbb_1;\cdots;\bbb_n] = [b_{i\kappa}]$. In \textsc{CompAlloc}, the input data set $\cD$ is partitioned to $\cD = \cup_{\kappa = 1}^{k} \cD_\kappa$ and distributed among the $n$ workers along with the corresponding coefficients. That is, each worker $i \in [n]$ is assigned with $\cD{(i)} = \cup_{\kappa = 1}^{k} b_{i\kappa} \cD_\kappa$ which specifies its local data set and corresponding combining coefficients. The parent of the $n$ workers is then able to recover the gradient over $\cD$, i.e. $\bg_{\cD}$ upon receiving the partial coded gradients of any $n-s$ workers and using the decoding matrix $\bA$ designed by  \textsf{GC} (Algorithm 1 in \cite{tandon2017gradient}). 



\noindent
\textbf{CodedReduce:} \textsf{CR} is implemented in two phases. It first allocates each worker with its local computation task via \textsc{CR.Allocate} procedure. This specifies each worker with its local data set and combining coefficients. Then, the communication strategy is determined by \textsc{CR.Execute}.

\begin{itemize}
    \item[] \textbf{\textsc{CR.Allocate}:}
    \begin{enumerate}
    \item Starting from the master, data set $  \cD^{T(1,i)}$ is assigned to sub-tree $T(1,i)$ for $i \in [n]$ via the allocation module \textsc{CompAlloc} (Figure \ref{fig:tree-data}).
    \item In layer $l=1$, each worker $(1,i)$, $i \in [n]$, picks $r_{\textsf{CR}}  d$ data points from the corresponding sub-tree's data set $\cD^{T(1,i)}$ as its local data set $\cD{(1,i)}$ and passes the rest $ \cD_{T(1,i)} = \cD^{T(1,i)} \setminus \cD(1,i)$ to its children and  their sub-trees (Figure \ref{fig:tree-data}). 
    \item Step (1) is repeated by using the module \textsc{CompAlloc} and treating $\cD_{T(1,i)}$ as the input data set to distribute it among the children of node $(1,i)$.
    \item Same procedure is applied till reaching the bottom layer (Figure \ref{fig:tree-data}). By doing so, the data set $\cD$ is redundantly distributed across the tree while all the workers are equally loaded with  $r_{\textsf{CR}}  d$ data points, where in Theorem \ref{thm:CRoptimality} we will show that $r_{\textsf{CR}} $ is a self-derived pick for \textsf{CR} given in (\ref{eq:rCR_lemma}).
    \end{enumerate}
    \item[] \textbf{\textsc{CR.Execute}:}
    \begin{enumerate}
        \item All the $N$ nodes start their local partial \emph{coded} gradient computations on the current model $\theta^{(t)}$, i.e.  $\bg_{\cD(l,i)}$ for all nodes $(l,i)$. Note that $\bg_{\cD(l,i)}$ is a coded gradient (i.e. a linear combination of partial gradients) since ${\cD(l,i)}$ carries combining coefficients along with its data points.
        \item Starting from the leaf nodes, they send their partial coded gradient computation results (messages) $\bbm_{(L,i)} = \bg_{\cD(L,i)}$ up to their parents.
        \item Upon receiving enough results from their children (any $n-s$ of them), workers in layer $L-1$ recover a linear combination of their children's messages via proper row in the decoding matrix $\bA$, e.g., parent node $(L-1,1)$ recovers from its children's messages $[ \bbm_{(L,1)};\cdots; \bbm_{(L,n)}]$ via the proper decoding row $\ba_{f(L-1,1)}$.
        \item Recovered partial gradient is added to the local partial coded gradient and is uploaded to the parent, e.g. node $(L-1,1)$  uploads $\bbm_{(L-1,1)}$ to its parent, where
        \begin{equation}
            \bbm_{(L-1,1)} = \ba_{f(L-1,1)} [ \bbm_{(L,1)};\cdots; \bbm_{(L,n)}] + \bg_{\cD(L-1,1)}. \nonumber
        \end{equation}
        \item The same procedure is repeated till reaching the master node which is able to aggregate the total gradient $\bg_{\cD}$.
    \end{enumerate}
\end{itemize}

\begin{figure}[t] 
    \centering
    \includegraphics[width=.45\textwidth]{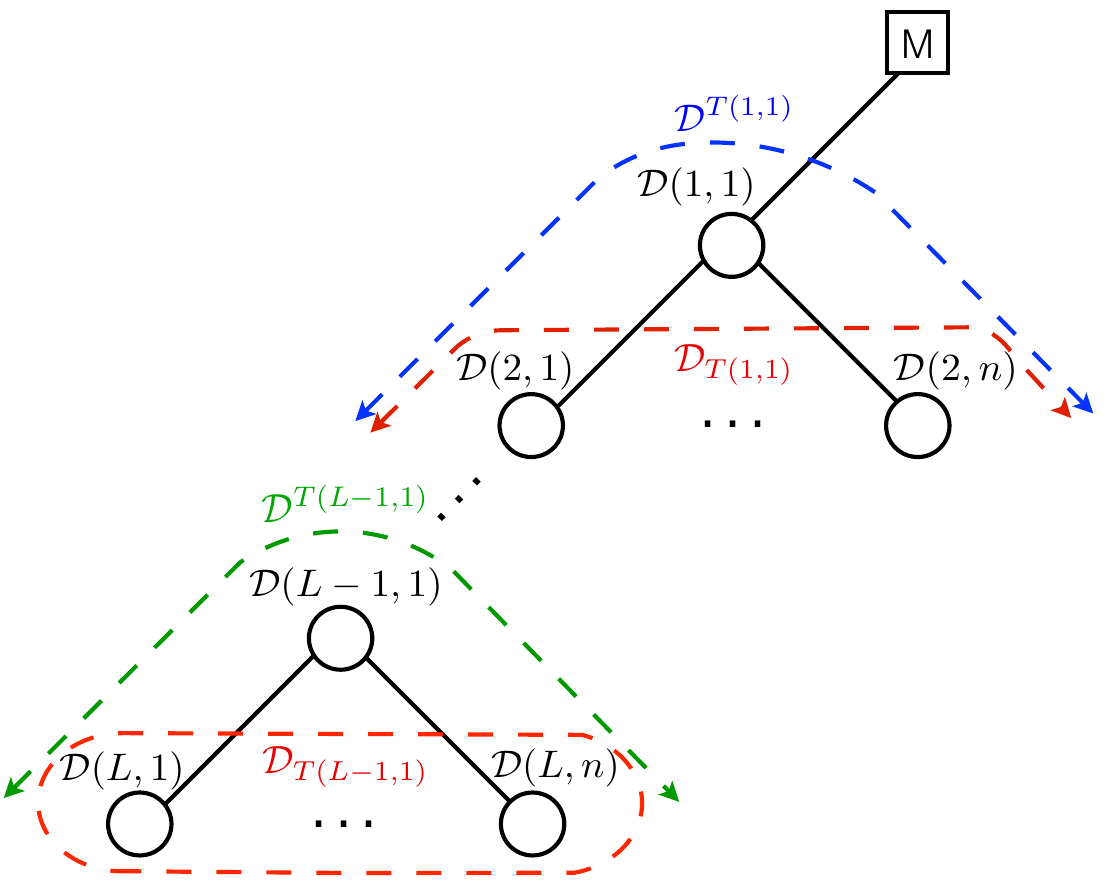}
    \caption{Illustration of  task allocation in \textsf{CR}.} \label{fig:tree-data}
\end{figure}

The pseudo-code for \textsf{CR} is available in Appendix \ref{appB}.


\subsection{An Example for \textsf{CR}}
In this section, we provide a simple example to better illustrate the proposed \textsf{CR} scheme.

\begin{example}[CodedReduce]
Consider a $(3,2)$--regular tree with $N=12$ nodes and $s=1$ straggler per parent. From \textsf{GC}, we have the decoding and encoding  matrices
\begin{equation}
\label{eq:gc3nodes}
    \bA = \begin{pmatrix}
        0 & 1 & 2 \\
        1 & 0 & 1 \\
        2 & -1 & 0
        \end{pmatrix},
\quad 
    \bB = \begin{pmatrix}
        1/2 & 1 & 0 \\
        0 & 1 & -1 \\
        1/2 & 0 & 1
        \end{pmatrix}.
\end{equation}

\begin{figure}[t] 
    \centering
    \includegraphics[width=.49\textwidth]{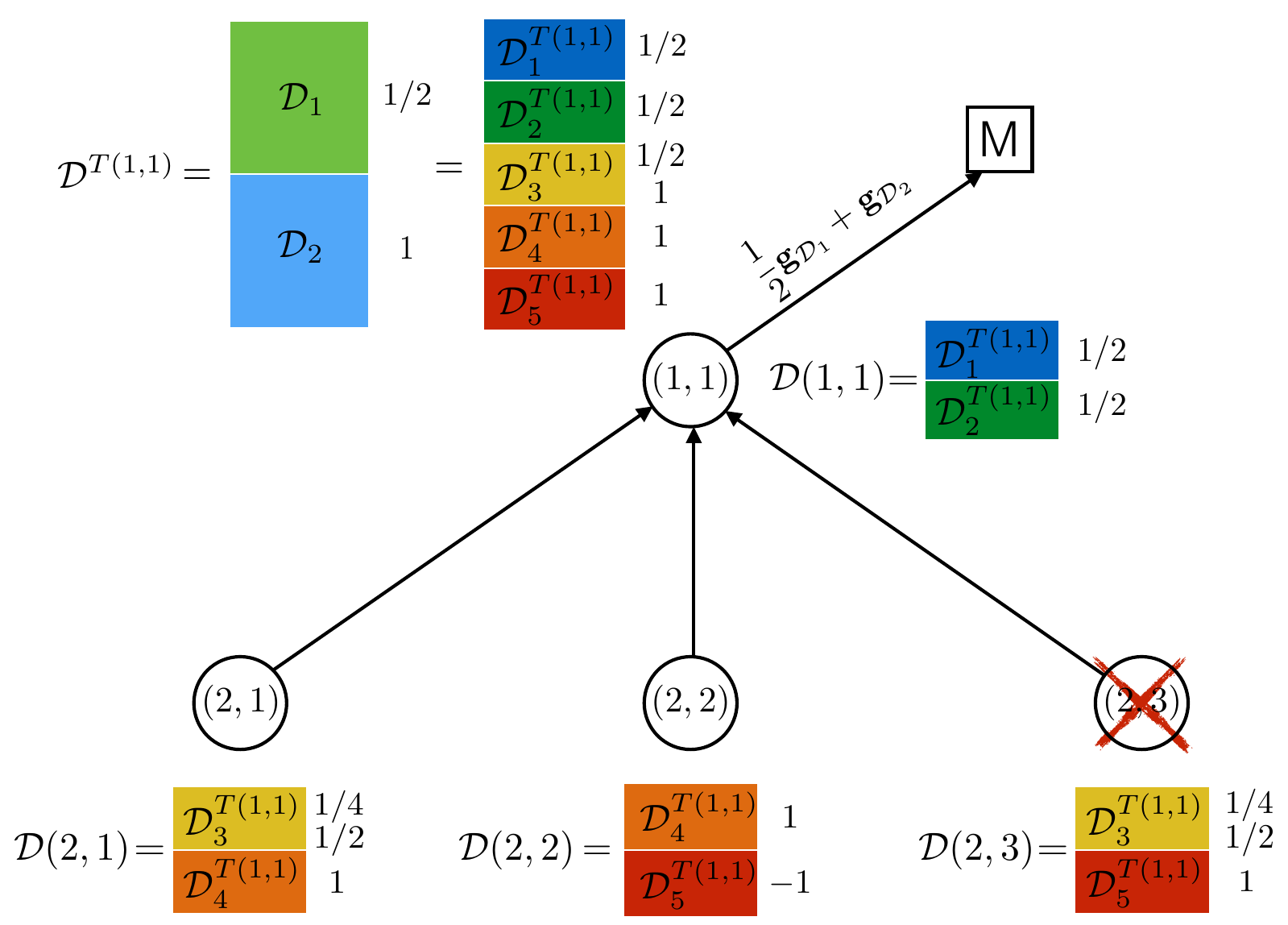}
    \caption{Illustration of data allocation and communication strategy in \textsf{CR} for a $(3,2)$--regular tree.} \label{fig:CR-ex}
\end{figure}

Following \textsf{CR}'s description,  we partition the data set of size $d$ as  $\cD=\{\cD_1,\cD_2,\cD_3\}$ and assign $\cD^{T(1,1)}=\frac{1}{2}\cD_1 \cup \cD_2$ to sub-tree $T(1,1)$. Node $(1,1)$ then picks $r_{\textsf{CR}} d= \frac{4}{15}d$ data points from $\cD^{T(1,1)}$ as $\cD{(1,1)}$. To do so, $\cD^{T(1,1)}$ is partitioned to $5$ sub-sets as $\cD^{T(1,1)} = \cD^{T(1,1)}_1 \cup \cdots \cup  \cD^{T(1,1)}_5$ and node $(1,1)$ picks the first two sub-sets, i.e. $\cD{(1,1)} = \cD^{T(1,1)}_1 \cup \cD^{T(1,1)}_2$ and the rest $\cD_{T(1,1)} = \cD^{T(1,1)}_3 \cup \cD^{T(1,1)}_4 \cup  \cD^{T(1,1)}_5$ is passed to layer $2$. Note that data points in $\cD{(1,1)}$ carry on the linear combination coefficients associated with $\cD^{T(1,1)}=\frac{1}{2}\cD_1 \cup \cD_2$. Figure \ref{fig:CR-ex} demonstrates each node  in sub-tree $T(1,1)$ with its allocated data set along with the encoding coefficients. Moving to layer $2$, $\cD_{T(1,1)}$ is partitioned to $3$ subsets and according to $\bB$ in (\ref{eq:gc3nodes}), the allocations to nodes $(2,1)$, $(2,2)$ and $(2,3)$ are as follows:
\begin{align}
    \cD{(2,1)} &= \frac{1}{2}\cD^{T(1,1)}_3 \cup \cD^{T(1,1)}_4, \nonumber\\
    \cD{(2,2)} &= \cD^{T(1,1)}_4 \cup (-1)\cD^{T(1,1)}_5, \nonumber\\
    \cD{(2,3)} &= \frac{1}{2}\cD^{T(1,1)}_3 \cup \cD^{T(1,1)}_5. \nonumber  
\end{align}
Similarly for other sub-trees, each node now is allocated with a data set for which each data point is associated with a scalar. For instance, node $(2,1)$ uploads $\bbm_{(2,1)} = \bg_{\cD(2,1)}=\frac{1}{2} \bg_{\cD^{T(1,1)}_3} + \bg_{\cD^{T(1,1)}_4}$
to its parent $(1,1)$. Node $(1,1)$ can recover from any $2$ surviving children, e.g. from $(2,1)$ and $(2,1)$ and using the first row in $\bA$, it uploads 
\begin{align}
    \bbm_{(1,1)} 
    &= [2,-1,0]  [\bbm_{(2,1)} ; \bbm_{(2,2)} ; \bbm_{(2,3)}] + \bg_{\cD(1,1)} \nonumber \\
    & = 2\bbm_{(2,1)} - \bbm_{(2,2)}  + \bg_{\cD(1,1)} \nonumber \\
    & = \frac{1}{2} \bg_{\cD_1} + \bg_{\cD_2} \nonumber
\end{align}
to the master. Similarly for other nodes, the master can recover the full gradient from any two children, e.g. using the second row of decoding matrix $\bA$ and surviving children $(1,1)$ and $(1,3)$:
\begin{align}
    [1,0,1] [\bbm_{(1,1)} ; & \bbm_{(1,2)}  ; \bbm_{(1,3)}] \nonumber \\
    &= \bbm_{(1,1)} + \bbm_{(1,3)} \nonumber \\
    & = \left(\frac{1}{2} \bg_{\cD_1} + \bg_{\cD_2} \right) + \left(\frac{1}{2} \bg_{\cD_1} + \bg_{\cD_3} \right) \nonumber \\
    & = \bg_{\cD}. \nonumber
\end{align}

\end{example}

\subsection{Theoretical Guarantees of \textsf{CR}}


In this section, we formally present the theoretical guarantees of \textsf{CR}. We first characterize the computation load induced by \textsf{CR} and demonstrate its significant improvement over \textsf{GC}. Then, we consider the commonly-used shifted exponential run-time computation distribution and a single-port communication model for workers  and asymptotically characterize the expected run-time of \textsf{CR} and conclude with  a discussion on its communication parallelization gain.


\textbf{Computation Load Optimality:} 
We show that for a fixed tree topology, the proposed \textsf{CR} is optimal in the sense that it achieves the minimum per-node computation load for a target resiliency. This optimality is established in two steps per  Theorem \ref{thm:CRoptimality}: (i) we first show the achievability by characterizing the computation load of \textsf{CR}; and (ii) we establish a converse showing that \textsf{CR}'s computation load is as small as possible. Proof is available in Appendix \ref{appC}.

\begin{theorem} \label{thm:CRoptimality}
For a fixed $(n,L)$--regular tree, any gradient aggregation scheme robust to any $s$ stragglers per any parent requires computation load $r$ where
\begin{equation}\label{eq:rCR_lemma}
    r \geq r_{\textsf{CR}} = \frac{1}{ \left( \frac{n}{s+1} \right) + \cdots + \left( \frac{n}{s+1} \right)^L }.
\end{equation}
\end{theorem}


\begin{remark}
While \textsf{CR} is $\alpha$-resilient, i.e. robust to \emph{any} $s=\alpha n$ stragglers per \emph{any} parent node, it significantly improves the per-node computation (and storage) load compared to an equivalent \textsf{GC} scheme with the same resiliency. In particular, \textsf{GC} loads each worker with $r_{\textsf{GC}} = \frac{S+1}{N} = \frac{\alpha N + 1}{N} \approx \alpha$ fraction of the data set, while \textsf{CR} considerably reduces it to 
$r_{\textsf{CR}} = 1 / \sum_{l=1}^L \left( \frac{n}{ \alpha n +1} \right)^l \approx \alpha^L$. For $\alpha=0.5$ as an instance, \textsf{CR} reduces the computation load $7 \times$ by rearranging the nodes from $1$ layer to $3$ layers.
\end{remark}

\begin{remark}
\textsf{CR} makes the distributed GD strategy $\alpha$-resilient, that is any $s=\alpha n$ stragglers per any parent node which sums up to a total of $S= \alpha N$ stragglers -- the same as the worst case number of stragglers in \textsf{GC}. It is clear than if the stragglers are picked adversarially, for instance all the nodes in layer $1$, then \textsf{CR} fails to recover the total gradient at the master. However, our experiments over Amazon EC2 confirm that stragglers are randomly distributed over the tree and not adversarially picked, which is aligned with the random stragglers pattern considered in this paper.
\end{remark}

\textbf{Total Gradient Computation Complexity:} To better characterize the advantages of \textsf{CR}, we characterize its \emph{total} gradient computation complexity in order to reach the final parameter model with predefined accuracy. More precisely, we focus on learning problems with strongly convex losses and let $T_{\text{\textsf{CR}}}$ denote the total number of iterations to reach a final model $\theta$ such that $\Vert \theta - \theta^* \Vert^2 \leq \epsilon$. Since in each iteration of \textsf{CR} the \emph{exact} gradient on all the $d$ data samples is computed (same as in GD), therefore $T_{\text{\textsf{CR}}} = \mathcal{O}(\log(1/\epsilon))$. In each iteration, each of the $N$ worker nodes compute $\alpha_{\text{\textsf{CR}}} \cdot d$ gradients, where according to Theorem \ref{thm:CRoptimality}, we have $\alpha_{\text{\textsf{CR}}} \approx \alpha^L$. All in all, in order to reach an $\epsilon$-accurate model, the \textsf{CR} method requires $\mathcal{O}(\alpha^L \cdot N \cdot \log(1/\epsilon) \cdot d)$ gradient computations in total.

One simple and yet naive approach to mitigate stragglers is to update the model using the gradient computation results of \emph{only} a fraction ($\alpha$) of worker node (non-stragglers). This approach can be treated as standard Stochastic Gradient Descent (SGD) which requires $T_{\text{SGD}} = \mathcal{O}(1/\epsilon)$ iterations in total to reach an $\epsilon$-accurate model.  Since each of the $N$ worker nodes store $d/N$ samples (i.e. no redundant data allocation), therefore  in each iteration, each node computes $\alpha d /N$ gradients. Putting all together, in order to reach $\epsilon$-accurate model, SGD requires $\mathcal{O}(\alpha  \cdot 1/\epsilon \cdot d)$ gradient computations in total. Comparing the two gradient computation complexities of \textsf{CR} and SGD, we observe that although SGD slashes the complexity by a \emph{linear} factor $N$, however, it suffer from two \emph{exponential} factors, that are growing $\alpha^L$ to $\alpha$ and $\log(1/\epsilon)$ to $1/\epsilon$ which significantly increase the total gradient computation complexities, as $\alpha^L \ll \alpha$ and $\log(1/\epsilon) \ll 1/\epsilon$.

\textbf{Latency Performance:} 
While we have derived the straggler resiliency of \textsf{CR}, the ultimate goal of a distributed gradient aggregation scheme is to have small latency which is partly attained by establishing higher communication parallelization. 

\textit{Computation Time Model}: We consider random computation time model for workers with shifted exponential distribution which is used in several prior works \cite{liang2014tofec,reisizadeh2017coded,li2018near}. More precisely, for a worker $W_i$ with assigned data set of size $d_i$, we model the computation time as a random variable with a shifted exponential distribution as follows:
\begin{equation}
\label{eq:comp_model}
\Prob[T_i\leq t]=1-e^{-\frac{\mu}{d_i}(t-a d_i)}, \text{ for }\,\,t\geq a d_i,
\end{equation}
where system parameters $a = \Theta(1)$ and $\mu = \Theta(1)$ respectively denote the shift and the exponential rate. We assume that $T_i$'s are independent. 

\textit{Communication Time Model}: To model the communication time and bandwidth bottleneck, we assume that each node is able to receive messages from only one other node at a time, and the total available bandwidth is dedicated to the communicating node. We also assume that communicating a partial gradient vector (of size $p$) from a child to its parent takes a constant time $t_c$.

The following theorem asymptotically characterizes the expected run-time of \textsf{CR} which we denote by $T_{\textsf{CR}}$ (Proof is available in Appendix \ref{appD}). More precisely, we consider the regime of interest where the data set size $d$ and the number of layers $L$ in the tree are fixed, while the number of children per parent, i.e. $n$ is approaching infinity with a constant straggler ratio $\alpha = s/n = \Theta(1)$.

\begin{theorem} \label{thm:CRtime}
Considering the computation time model in (\ref{eq:comp_model}) for workers, the expected run-time of \textsf{CR} on an $(n,L)$--regular tree with resiliency $\alpha = \Theta(1)$ satisfies the followings:
\begin{align}
    \Expc \left[T_{\textsf{CR}} \right] 
    &\geq
    \frac{r_{\textsf{CR}} d}{\mu} \log \left( \frac{1}{\alpha} \right) + a r_{\textsf{CR}} d \nonumber\\
    &  
    + \left( n(1-\alpha)-o(n)+L-1 \right) \left( 1-o(1) \right) t_c + o(1), \label{eq:lowerb} \\
    \Expc \left[T_{\textsf{CR}} \right] 
    &\leq 
    \frac{r_{\textsf{CR}} d}{\mu} \log \left( \frac{1}{\alpha} \right) + a r_{\textsf{CR}} d 
    \nonumber\\
    & \quad
    +  n \left( 1-o(1) \right) L t_c + o(1) \label{eq:upperb}. 
\end{align}
\end{theorem}
\begin{remark}
Theorem \ref{thm:CRtime} implies that the expected run-time of the proposed \textsf{CR} algorithm breaks down into two terms: $\Expc \left[T_{\textsf{CR}} \right] = \Theta(1) + \Theta(n)$, where the two terms  $\Theta(1)$ and $\Theta(n)$ correspond to computation and communication times, respectively. As a special case, it also implies that the average run-time for \textsf{GC} is $\Expc \left[T_{\textsf{GC}} \right] = \Theta(1) + \Theta(N)$. This clearly demonstrates that \textsf{CR} is indeed alleviating the bandwidth bottleneck and it improves the communication parallelization gain from $\beta_{\textsf{GC}} = \Theta(1)$ to $\beta_{\textsf{CR}} = \Theta(N/n) = \Theta(N^{1-1/L})$ by parallelizing the communications over an $L$-layer tree structure. 
\end{remark}

%% file: 6-experiments.tex
In this section, we provide the results of our experiments conducted over Amazon EC2, for which we used \texttt{Python} with \texttt{mpi4py} package. Our results demonstrate significant speedups of \textsf{CR} over baseline approaches. We consider two sets of machine learning experiments, one with a real data set, and another with an artificial data set. For each machine learning setting, we consider two cluster configurations, one with $N=84$ workers, and another with $N=156$ workers, using \texttt{t2.micro} instance for master and all workers. Furthermore, each experiment is run for $300$ rounds. Next, we describe the experiments in detail and provide the results.

\subsection{Convex Optimization}
\subsubsection{Real Data Set}
We consider the machine learning problem of logistic regression via gradient descent (GD) over the real data set GISETTE \cite{guyon2007competitive}. The problem is to separate the often confused digits `9' and `4'. We use $d=6552$ training samples, with model size $p=5001$. The following relative error rate is considered for model estimation:
\begin{equation}
\label{eq:RER}
    \text{Relative Error Rate } = \frac{\norm{ \theta^{(t)}-\theta^{(t-1)}}^2}{\norm{\theta^{(t-1)}}^2},
\end{equation}
where $\theta^{(t)}$ denotes the estimated model at iteration $t$. The following schemes are considered for data allocation and gradient aggregation:
\begin{figure}[h!]
            \centering \includegraphics[width= 0.47 \textwidth]{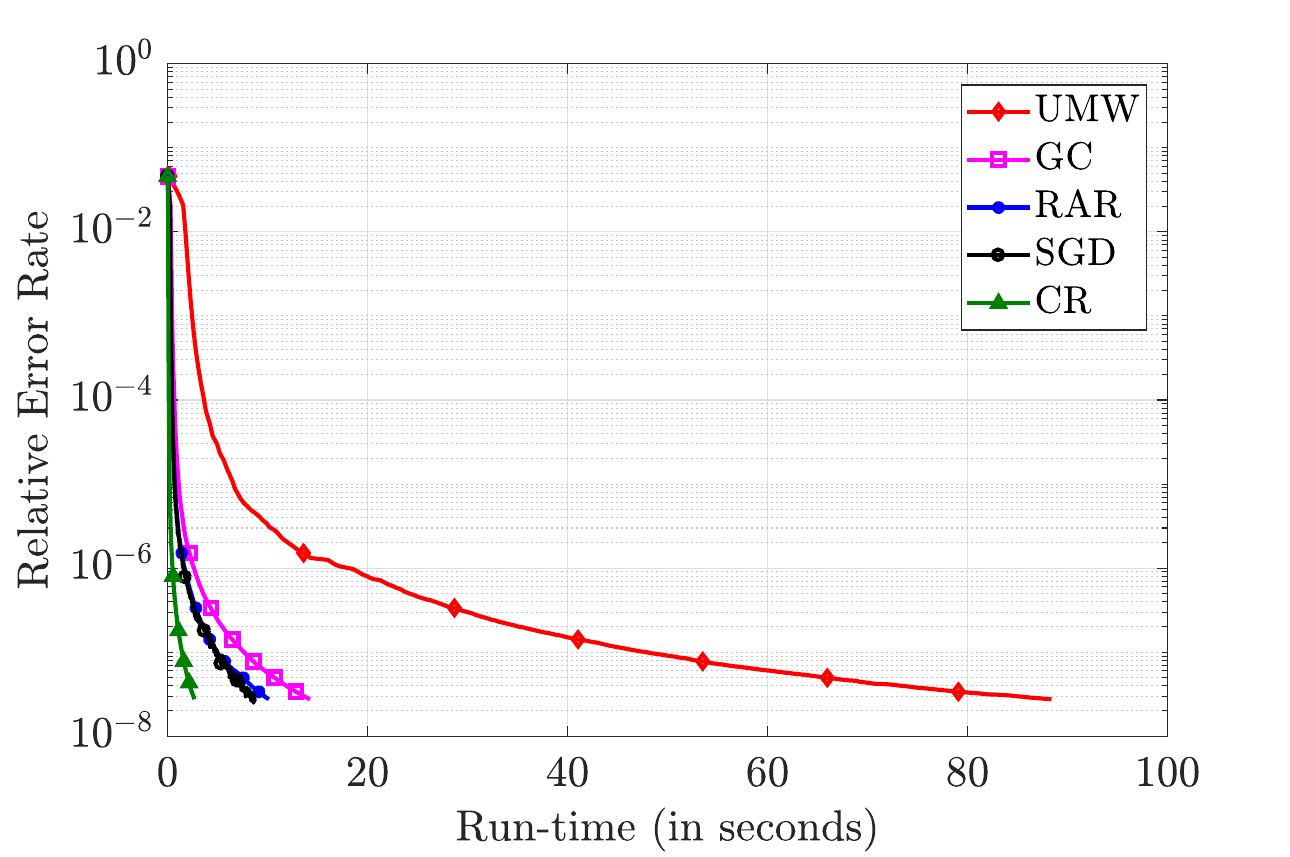}
        \caption{Convergence curves for relative error rate vs wall-clock time for logistic regression over $N=84$ workers. The straggler resiliency is $\alpha=1/4$. \textsf{CR} achieves a speedup of up to $32.8\times$, $5.3\times$, $3.8\times$ and $3.2\times$ respectively over \textsf{UMW}, \textsf{GC}, \textsf{RAR} and \textsf{SGD}.} 
  \label{fig:log84}
\end{figure} 

\begin{enumerate}
\item {Uncoded Master-worker (\textsf{UMW})}: This is the naive scheme in which the data set is uniformly partitioned among the workers, and the master waits for results from all the workers to aggregate the gradient. 

\item {Gradient Coding (\textsf{GC})}: We implement \textsf{GC} as described in Section \ref{subsec:GC}, with the straggler parameter $S=\alpha N$.

\item {Ring-AllReduce (\textsf{RAR})}: The data set is uniformly partitioned over the workers and the MPI function \texttt{MPI\_Allreduce()} is used for gradient aggregation.

\item {Stochastic Gradient Descent (\textsf{SGD})}: The data allocation is the same as \textsf{UMW}. However, the master updates the model using the partial gradient obtained via aggregating the results from results of \textit{only} the first $N-S$ children. Furthermore, as is typical in SGD experiments, we used a learning rate of $c_1/(t + c_2)$ where $c_1$ and $c_2$ were numerically optimized. 

\item {CodedReduce (\textsf{CR})}: We implement our proposed scheme  as presented in  Section \ref{sec:codedreduce} on a tree with $(n,L)=(12,2)$, while the straggler parameter $s=\alpha n$.

\end{enumerate}

\begin{figure}[h!]
    \centering
        \begin{subfigure}[h]{0.47\textwidth}
         \centering \includegraphics[width= 1.0 \textwidth]{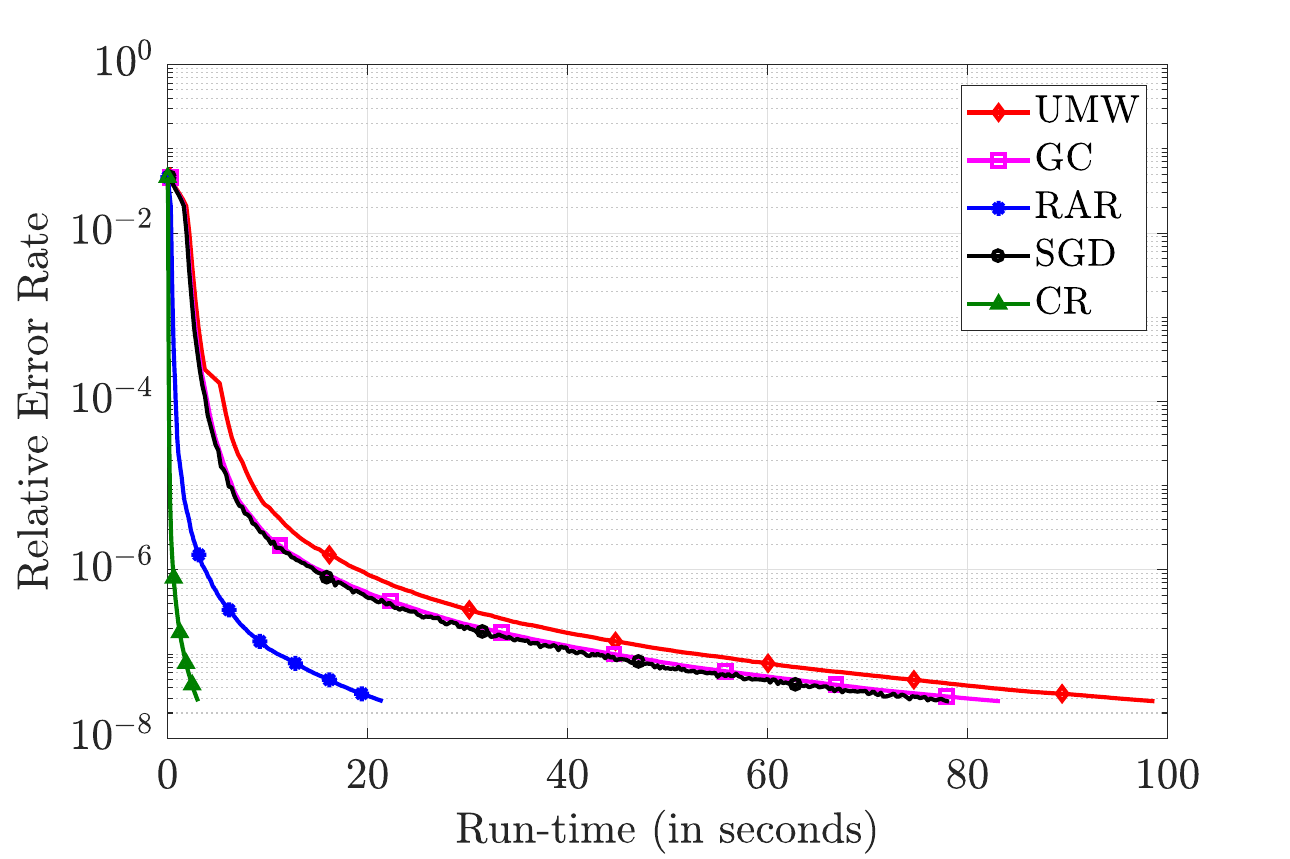}
        \caption{Convergence curves  for $\alpha=1/12$. \textsf{CR} achieves a speed up of up to $32.3\times$, $27.2\times$, $7.0\times$ and $25.4\times$ respectively over \textsf{UMW}, \textsf{GC}, \textsf{RAR} and \textsf{SGD}.} 
  \label{fig:log156_1by12}
    \end{subfigure}
     ~ 
    \begin{subfigure}[h]{0.47\textwidth}
        \centering \includegraphics[width= 1.0 \textwidth]{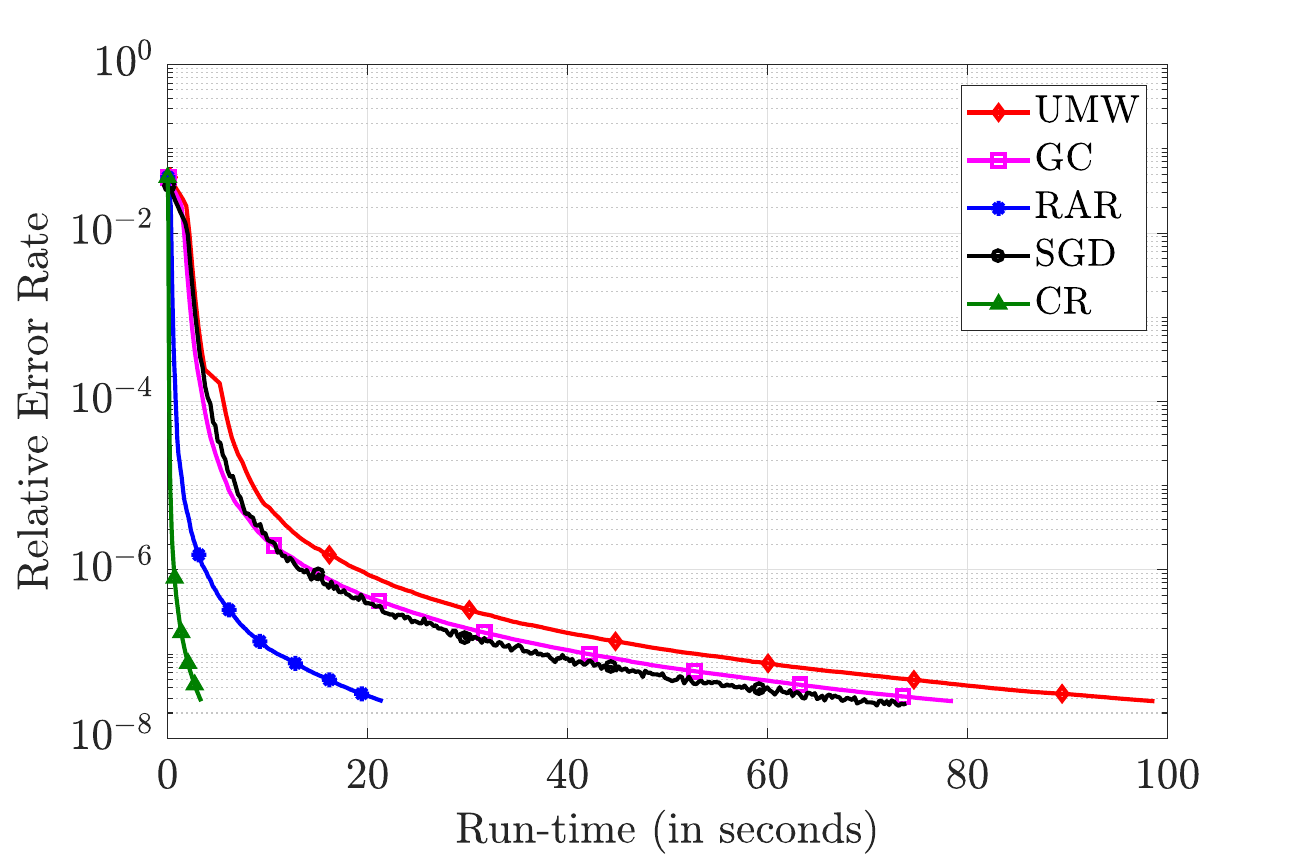}
        \caption{Convergence curves  for $\alpha=2/12$. \textsf{CR} achieves a speed up of up to $29.3\times$, $23.3\times$, $6.4\times$ and $21.9\times$ respectively over \textsf{UMW}, \textsf{GC}, \textsf{RAR} and \textsf{SGD}.} 
  \label{fig:log156_2by12}
    \end{subfigure}
    ~
    \begin{subfigure}[h]{0.47\textwidth}
        \centering \includegraphics[width= 1.0 \textwidth]{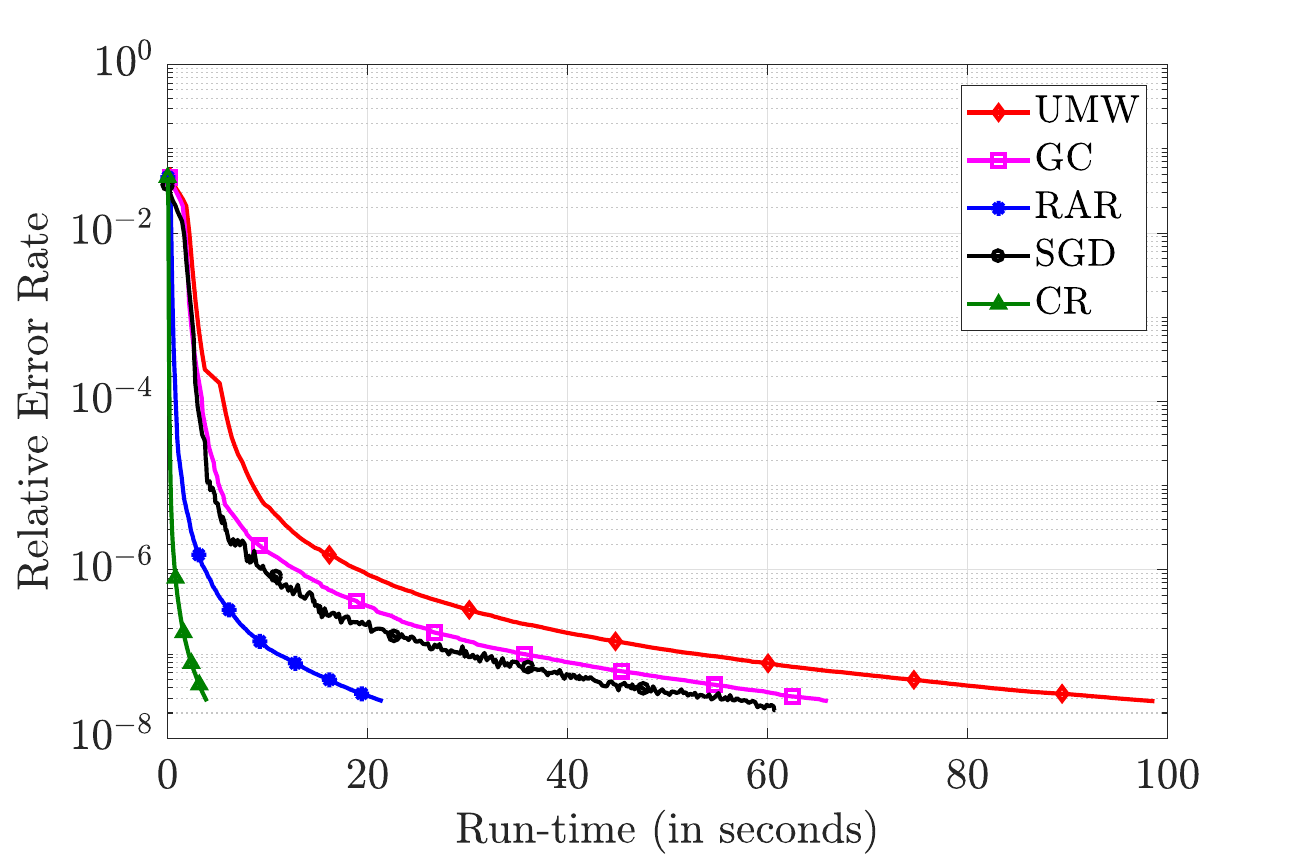}
        \caption{Convergence curves  for $\alpha=3/12$. \textsf{CR} achieves a speed up of up to $25.0\times$, $16.8\times$, $5.4\times$ and $15.4\times$ respectively over \textsf{UMW}, \textsf{GC}, \textsf{RAR} and \textsf{SGD}.} 
  \label{fig:log156_3by12}
    \end{subfigure}
    ~
    \caption{Convergence results for relative error rate vs wall-clock time for logistic regression over $N=156$ workers with different straggler resiliency $\alpha$.}
  \label{fig:log156}
\end{figure}

Next, we plot the relative error rate defined in (\ref{eq:RER}) as a function of wall-clock time for our logistic regression experiments with $N=84$ workers and $N=156$ workers respectively in Fig. \ref{fig:log84} and Fig. \ref{fig:log156}. For $N=84$, we consider a straggler-resiliency of $\alpha = 1/4$, while for $N=156$, we consider three different values of $\alpha: 1/12, 2/12 \text{ and } 3/12$. 

We make the following observations from the plots:

\begin{itemize}
    \item As demonstrated by Fig. \ref{fig:log84} and \ref{fig:log156}, \textsf{CR} achieves significant speedups over the baseline approaches. Specifically, for $(N,\alpha) = (84,1/4)$, \textsf{CR} is faster than \textsf{UMW}, \textsf{GC}, \textsf{RAR} and \textsf{SGD} by $32.8\times$, $5.3\times$, $3.8\times$ and $3.2\times$ respectively. For $(N,\alpha) = (156,1/12)$,
    \textsf{CR} achieves speedups of 
    $32.3\times$, $27.2\times$, $7.0\times$ and $25.4\times$ respectively over \textsf{UMW}, \textsf{GC}, \textsf{RAR} and \textsf{SGD}. Similar speedups are obtained with $(N,\alpha) = (156,2/12)$ and $(N,\alpha) = (156,3/12)$, as demonstrated by Fig. \ref{fig:log156}(\subref{fig:log156_2by12}) and Fig. \ref{fig:log156}(\subref{fig:log156_3by12}) respectively.
    \item Although \textsf{GC} gains over \textsf{UMW} by avoiding stragglers, its performance is still bottlenecked by bandwidth  congestion, and the increase in  computation load at each worker by a factor of $(S+1)$ in comparison to \textsf{UMW}. The bottlenecks are reflected in comparison with \textsf{SGD}, which has similar or better performance in comparison to \textsf{GC} due to much less computation load per worker. 
    \item \textsf{RAR} significantly outperforms \textsf{UMW} as well as \textsf{GC} for $N=84$ as well as $N=156$ worker settings. Although \textsf{RAR} achieves similar performance in comparison to \textsf{SGD} for $N=84$ workers scenario, it ultimately  beats all the schemes with the generic master-worker topology when the cluster size is increased to $N=156$. Our proposed \textsf{CR} algorithm combines the best of \textsf{GC} and \textsf{RAR} by providing straggler robustness via coding and alleviating bandwidth bottleneck via a tree topology.
\end{itemize}

\subsubsection{Artificial Data Set}

\begin{figure}[h!]
            \centering \includegraphics[width= 0.47 \textwidth]{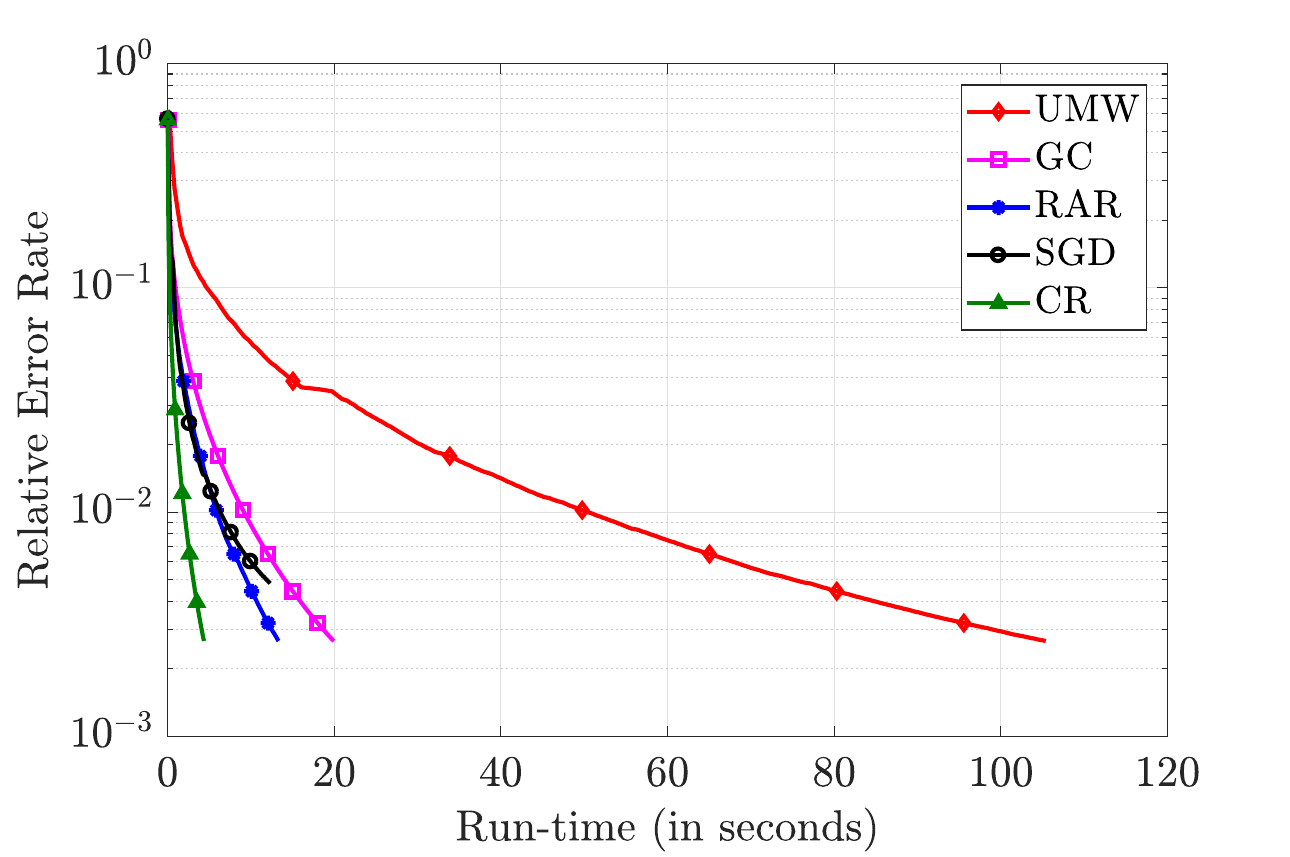}
        \caption{Convergence curves for normalized error rate vs wall-clock time for linear regression over $N=84$ workers. The straggler resiliency is $\alpha=1/4$. \textsf{CR} achieves a speedup of up to $24.1\times$, $4.6\times$, $3.0\times$ and $2.8\times$ respectively over \textsf{UMW}, \textsf{GC}, \textsf{RAR} and \textsf{SGD}.} 
  \label{fig:lin84}
\end{figure}
Next we solve a linear regression problem via GD over a synthetic data set with parameters $(d,p)=(7644,6500)$. We generate the data set using the following model:
\begin{equation}
    \bx_{j}(p+1) = \bx_{j}(1:p)^{\top} \theta_* + z_j, \quad \text{ for } j \in [d],
\end{equation}
where the true model $\theta_*$ and features $\bx_{j}(1:p) = [\bx_{j}(1);\cdots;\bx_{j}(p)]$ are drawn randomly from $\mathcal{N}(0,I_p)$ distribution and $z_j$ is  a standard Gaussian noise. We consider the following normalized error rate:
\begin{equation}
\label{eq:NER}
    \text{Normalized Error Rate} = \frac{\norm{ \theta^{(t)}-\theta_*}^2}{\norm{\theta_*}^2}.
\end{equation}

\begin{figure}[h!]
    \centering
        \begin{subfigure}[h]{0.47\textwidth}
         \centering \includegraphics[width= 1.0 \textwidth]{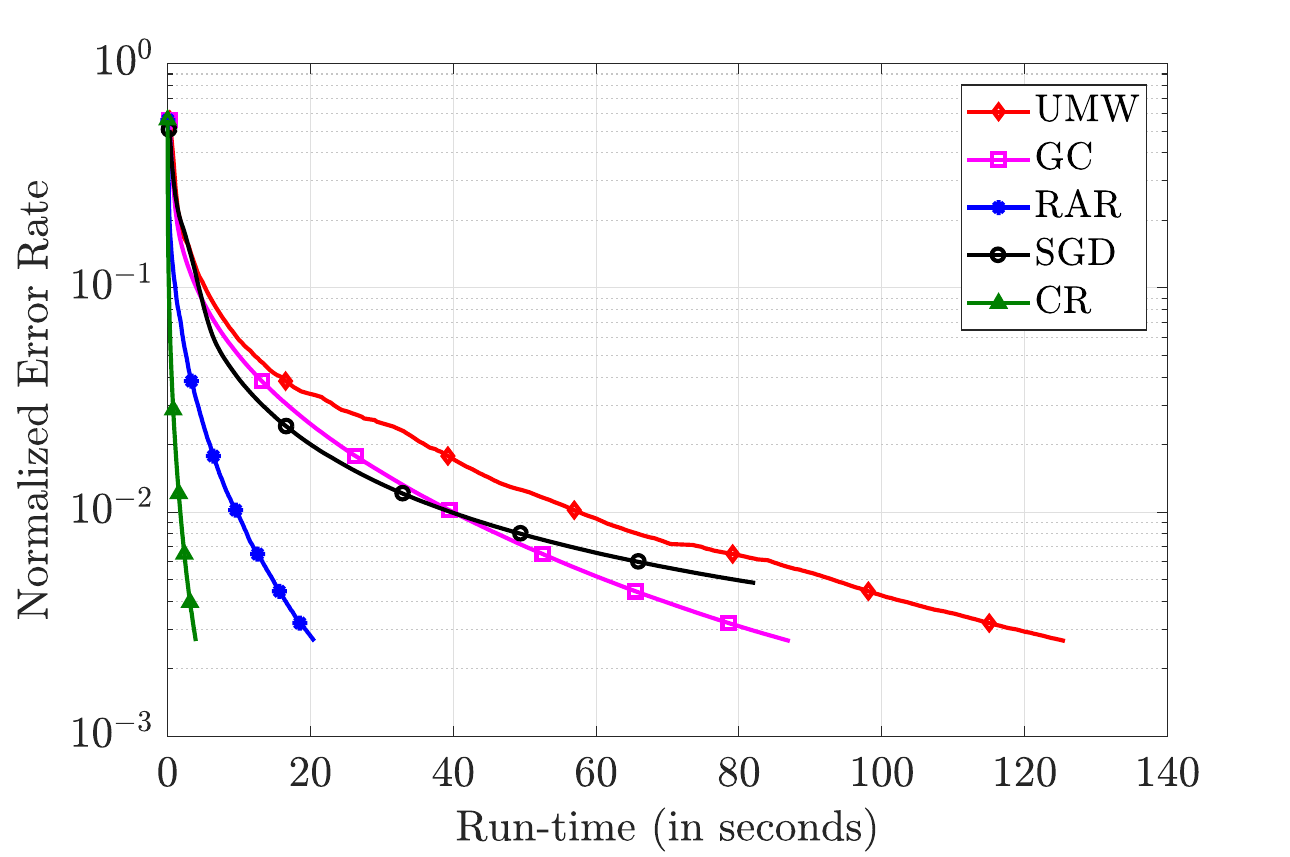}
        \caption{Convergence curves  for $\alpha=1/12$. \textsf{CR} achieves a speed up of up to $31.7\times$, $22.0\times$, $5.2\times$ and $20.7\times$ respectively over \textsf{UMW}, \textsf{GC}, \textsf{RAR} and \textsf{SGD}.} 
  \label{fig:lin156_1by12}
    \end{subfigure}
     ~ 
    \begin{subfigure}[h]{0.47\textwidth}
        \centering \includegraphics[width= 1.0 \textwidth]{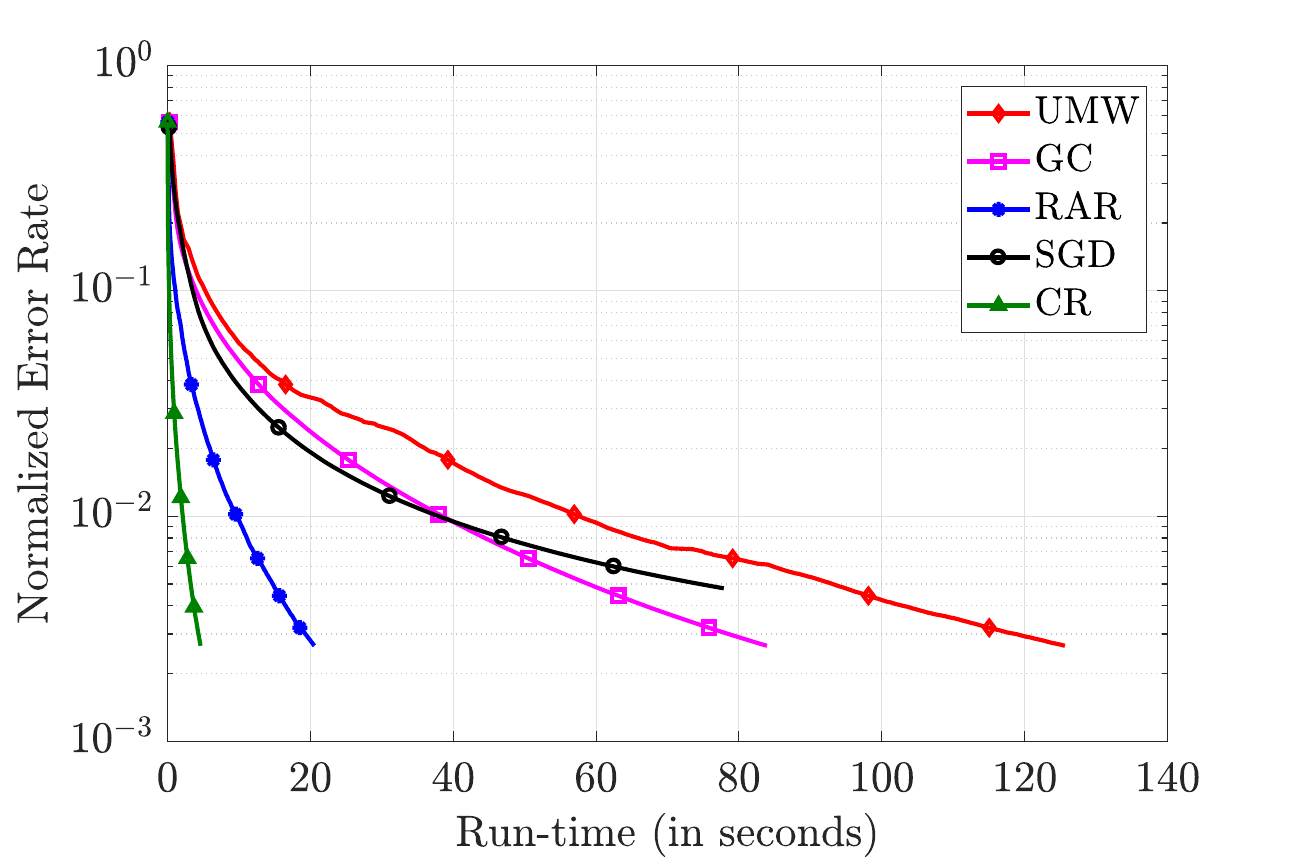}
        \caption{Convergence curves  for $\alpha=2/12$. \textsf{CR} achieves a speed up of up to $27.1\times$, $18.1\times$, $4.4\times$ and $16.8\times$ respectively over \textsf{UMW}, \textsf{GC}, \textsf{RAR} and \textsf{SGD}.} 
  \label{fig:lin156_2by12}
    \end{subfigure}
    ~
    \begin{subfigure}[h]{0.47\textwidth}
        \centering \includegraphics[width= 1.0 \textwidth]{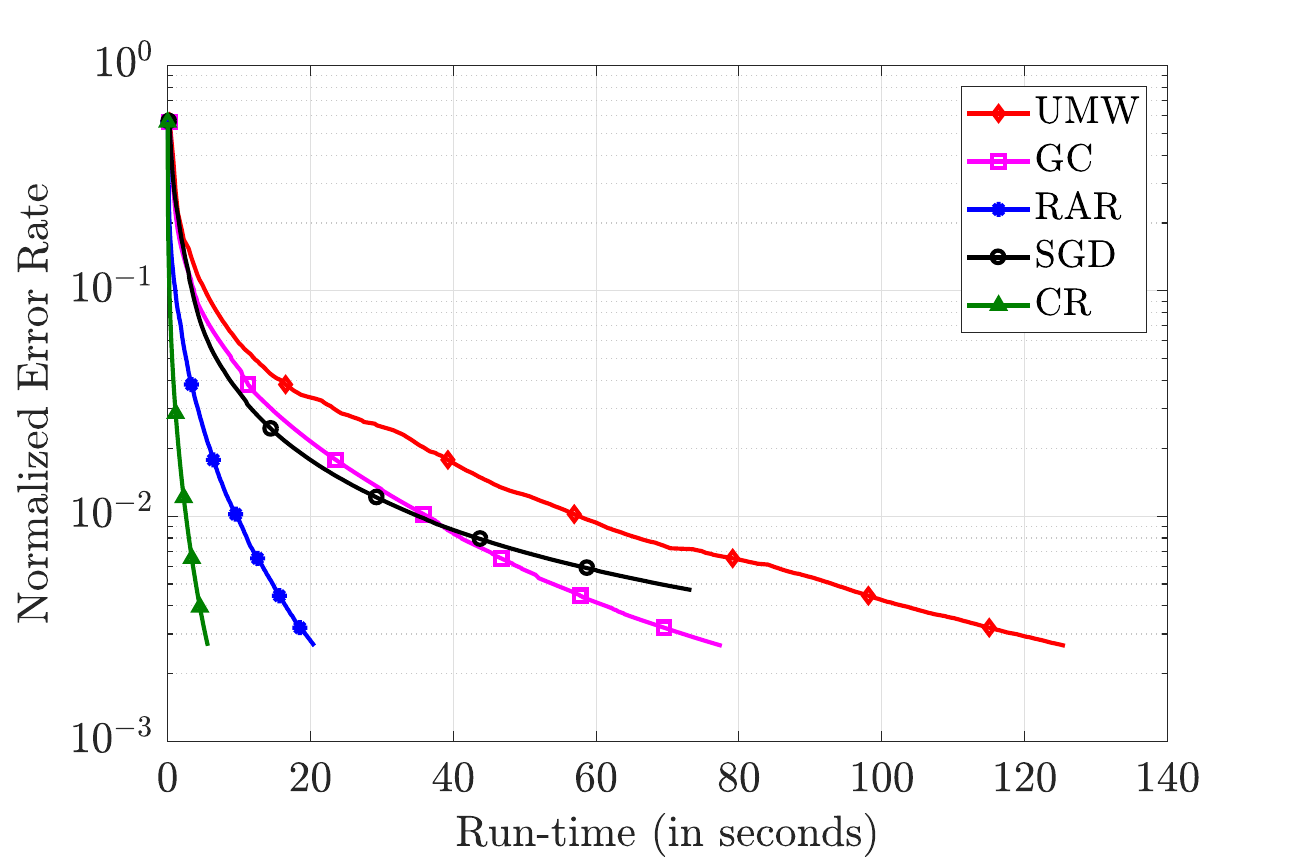}
        \caption{Convergence curves  for $\alpha=3/12$. \textsf{CR} achieves a speed up of up to $22.2\times$, $13.7\times$, $3.6\times$ and $13.0\times$ respectively over \textsf{UMW}, \textsf{GC}, \textsf{RAR} and \textsf{SGD}.} 
  \label{fig:lin156_3by12}
    \end{subfigure}
    ~
    \caption{Convergence results for normalized error rate vs wall-clock time for linear regression over $N=156$ workers with different straggler resiliency $\alpha$.}
  \label{fig:lin156}
\end{figure}


In Fig. \ref{fig:lin84} and \ref{fig:lin156}, we plot the normalized error rate defined in (\ref{eq:NER}) as a function of wall-clock time for $N=84$ and $N=156$ respectively. We consider similar configuration and schemes as for the experiments with real data set. The following observations are made with regard to the experiments:

\begin{itemize}
    \item As in the previous case of logistic regression with real data set, \textsf{CR} achieves significant speedups over baseline approaches for linear regression as well. Particularly, for $(N,\alpha)=(84,1/4)$, \textsf{CR} achieves speedups of $24.1\times$, $4.6\times$, $3.0\times$ and $2.8\times$ over \textsf{UMW}, \textsf{GC}, \textsf{RAR} and \textsf{SGD} respectively. When $(N,\alpha) = (156,1/12)$, \textsf{CR} achieves speedups of $31.7\times$, $22.0\times$, $5.2\times$ and $20.7\times$ in comparison to \textsf{UMW}, \textsf{GC}, \textsf{RAR} and \textsf{SGD} respectively. Similar speedups are obtained for $(N,\alpha) = (156,2/12)$ and $(N,\alpha)=(156,3/12)$. 
    \item \textsf{GC} performs better than \textsf{UMW} by avoiding stragglers. However, its performance is still bottlenecked by bandwidth congestion and the increase in computation load at each worker by a factor of $(S+1)$ in comparison to \textsf{UMW}. 
    \item \textsf{SGD} achieves a gain in per iteration time over \textsf{UMW} and \textsf{GC}. However, it has higher normalized error with respect to the true model. 
    \item Combined with the results of logistic regression, our experiments complement the theoretical gains of \textsf{CR} that have been established earlier. As demonstrated by the results, a tree-based topology is well-suited for bandwidth bottleneck alleviation in large-scale commodity clusters. Furthermore, the data allocation and coding strategy provide resiliency to stragglers. 
    \end{itemize}

\begin{figure}[h!]
            \centering \includegraphics[width= 0.47 \textwidth]{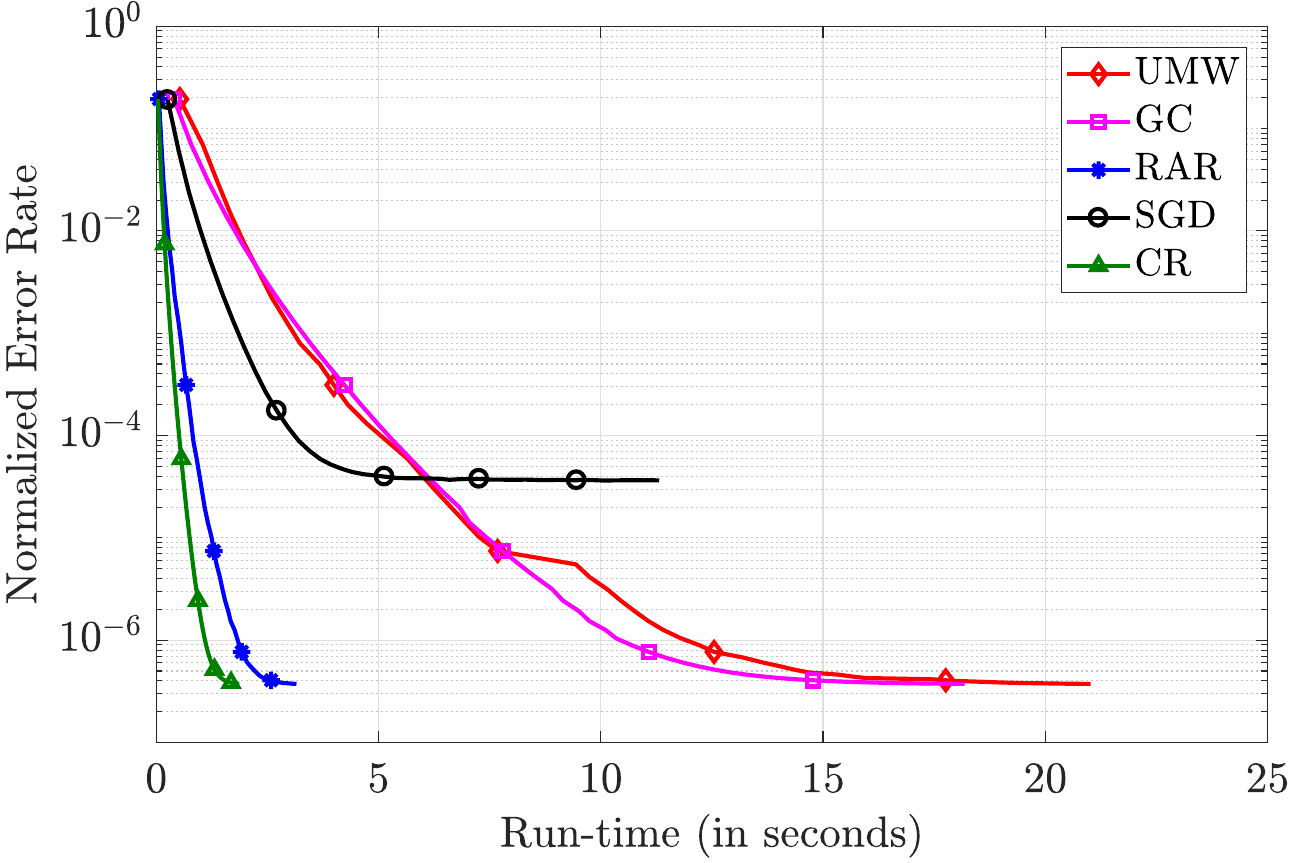}
        \caption{Convergence curves for normalized error rate vs wall-clock time for linear regression over $N=156$ workers and $(d,p)=(32760,5000)$. The straggler resiliency is $\alpha=1/4$ and the number of rounds is $50$.  \textsf{CR} achieves a speedup of up to $11.3\times$, $9.7\times$, $1.69\times$ and $6.1\times$ respectively over \textsf{UMW}, \textsf{GC}, \textsf{RAR} and \textsf{SGD}.} 
  \label{fig:large_exp}
\end{figure} 

\begin{remark}
    Till now, we have considered small-scale datasets in our experiments, which is motivated by the fact that in edge based devices with non-dedicated resources, the amount of memory available for computation shall be low. Nevertheless, our proposed scheme \textsf{CR} can speedup general machine learning in cloud environments. To illustrate this point, we have carried out another experiment with a larger dataset $(d,p) = (32760,500)$, with $(N,\alpha)=(156,1/4)$. As illustrated by Fig. \ref{fig:large_exp}, \textsf{CR} outperforms the baseline approaches by considerable margins. Specifically, \textsf{CR} achieves a speedup of $11.3\times, 9.7\times, 1.69\times$ and $6.1\times$ over \textsf{UMW}, \textsf{GC}, \textsf{RAR} and \textsf{SGD} respectively.   
\end{remark}

\subsection{Neural Networks}
We carry out simulations for evaluating the benefits of \textsf{CR} in distributed training of neural networks with cross-entropy loss, which essentially involves non-convex and non-smooth loss functions due to variety of non-linearities such as ReLUs. For this, we consider the CIFAR10 dataset \cite{cifar10dataset}, which has $10$ different categories of images. CIFAR10 has $50000$ images while the test dataset has $10000$ images. We provide the details of the neural network in Table \ref{tab:nn_arch}. We use an initial step size of $0.02$, and a step decay of $0.7$ at iterations $1300$ and $2100$. We use Glorot uniform initializer for initializing the convolutional layer weights, while for fully connected layers, we use the default initializer. 

We consider a cluster of $N=156$ servers, a resiliency of $5/12$, and $n=12$ children per node for \textsf{CR}. We use a random subset of $d=49920$ training images for training. Accuracy is reported on test dataset. We use the Pytorch library for neural network training. Furthermore, we use the computation and communication model as described earlier, where we assume $t_c=0.05$ seconds,  $a=5{\times}10^{-5}$ seconds/data, and assume $a\mu=1$. 

In Fig. \ref{fig:nn_exp}, we plot the accuracy vs wall-clock time curves for the different approaches, where training is carried out for a total of $2500$ iterations. Clearly, \textsf{CR} outperforms other approaches by significant margins. Particularly, \textsf{CR} achieves a speedup of up to $6.6\times$, $4.8\times$, $1.8\times$ and $4.0\times$ respectively over \textsf{UMW}, \textsf{GC}, \textsf{RAR} and \textsf{SGD}.   

\begin{figure}[h!]
            \centering \includegraphics[width= 0.47 \textwidth]{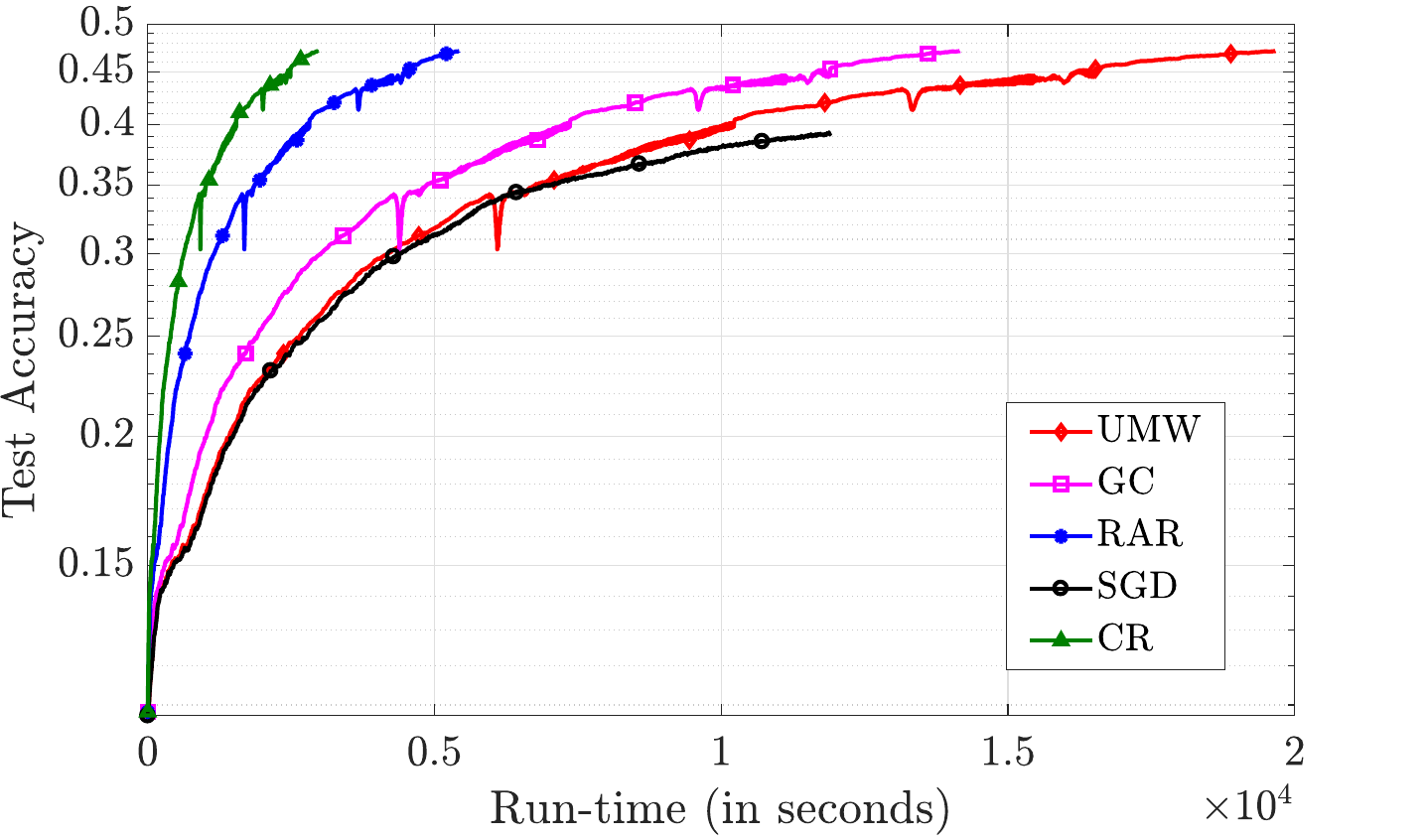}
        \caption{Convergence curves for test accuracy vs wall-clock time for neural network training over $N=156$ workers. The neural network model has $p\approx 120,000$  parameters. The straggler resiliency is $\alpha=5/12$ and the number of rounds is $2500$.  \textsf{CR} achieves a speedup of up to $6.6\times$, $4.8\times$, $1.8\times$ and $4.0\times$ respectively over \textsf{UMW}, \textsf{GC}, \textsf{RAR} and \textsf{SGD}.} 
  \label{fig:nn_exp}
\end{figure}

\begin{table}[htb!]
\caption{Details of the neural network architecture used in the simulations.}\label{tab:nn_arch}
\centering
\begin{tabular}{|l|l|l|l|}
\hline
\textbf{Sl. No.} & \textbf{Parameter} & \textbf{Shape} & \textbf{Hyperparameters} \\ \hline
$1$ & Conv2d& $3{\times}16{\times}3{\times}3$& stride$=1$, padding$=(1, 1)$ \\ \hline
$2$ & Conv2d& $16{\times}64{\times}4{\times}4$& stride$=1$, padding$=(0, 0)$ \\ \hline
$3$ & Linear& $64{\times}384$&-\\ \hline
$4$ & Linear& $384{\times}192$&-\\ \hline
$5$ & Linear& $192{\times}10$&-\\ \hline
\end{tabular}
\end{table}

%% file: 7-conclusion.tex
To conclude, we discussed two critical bottlenecks in scaling up Gradient Descent-based distributed learning frameworks: communication efficiency and stragglers' delays. We proposed CodedReduce (\textsf{CR}), that is a joint communication topology design and data set allocation strategy. \textsf{CR} combines the best of two existing approaches--Ring-AllReduce (\textsf{RAR}) and Gradient Coding (\textsf{GC})--by leveraging communication parallelization of \textsf{RAR} and straggler resiliency of \textsf{GC}. Theoretically, we characterized the computation load and straggler resiliency of \textsf{CR} and its asymptotic expected run-time. Lastly,  we  empirically  demonstrated that our proposed \textsf{CR} design achieves speedups of up to $27.2\times$ and $7.0\times$, respectively  over  the  \textsf{GC} and \textsf{RAR}.

We also discussed that although the main goal in the proposed \textsf{CR} design is to recover the \emph{exact} total gradient in each iteration of GD, one can relax this goal to inexact gradient aggregation leading to SGD-type optimization methods. We discussed how straggler resiliency and communication efficiency in GD-type methods can be improved by employing the \textsf{CR} design, while requiring lower computation complexity compared to naive SGD-type procedures. We note that although SGD has been widely considered for large-scale training, GD is still the prominent choice in many industry settings where one wants to make sure that the gradient computations are done completely so as not to lose even a little bit of performance. This is very critical since the model will be used by millions of people and even a slight improvement by GD would be useful. We note that \textsf{CR} may not be applicable in SGD settings in its current fashion. The reason is that the whole coded task allocation and execution described in the proposed \textsf{CR} algorithm is for the purpose of \emph{exact} gradient recovery, i.e. GD. Such elaborate and extra gradient computation makes less sense if we relax our goal to inexact gradient recover, i.e. SGD. There are simple and complexity efficient approaches to deal with stragglers in SGD settings, such as wait for $\alpha$ fraction of nodes to respond, as explained in Sections \ref{sec:codedreduce} and \ref{sec:experiments}. It is yet an interesting future direction to study potential coding opportunities for straggler mitigation in SGD scenarios.

Lastly, the tree structure proposed in this paper opens up new interesting directions in order to further improve the resiliency of distributed gradient aggregation schemes. For instance, given a fix set of available worker nodes, how can one find the optimal tree (i.e. optimal depth and width) in order to minimize the expected run-time.

%% file: 8-bios.tex
\begin{IEEEbiographynophoto}{Amirhossein Reisizadeh}
received his B.S. degree form Sharif University of Technology, Tehran, Iran in 2014 and an M.S. degree from University of California, Los Angeles (UCLA) in 2016, both in Electrical Engineering. He is currently pursuing his Ph.D. in  Electrical and Computer Engineering at University of California, Santa Barbara (UCSB). He was a finalist in the Qualcomm Innovation Fellowship program in 2019. He is interested in using information and coding-theoretic concepts to develop fast and efficient algorithms for large-scale machine learning, distributed computing and optimization.
\end{IEEEbiographynophoto}

\begin{IEEEbiographynophoto}{Saurav Prakash} 
received the Bachelor of Technology degree in Electrical Engineering from the Indian Institute of Technology (IIT), Kanpur, India in 2016. He is currently pursuing the Ph.D. degree in Electrical and Computer Engineering at the University of Southern California (USC), Los Angeles. His research interests include information theory and data analytics with applications in large-scale machine learning and edge computing. He is one of the recipients of the  Qualcomm Innovation Fellowship 2021. Saurav also received the Annenberg Graduate Fellowship in 2016 and was one of the Viterbi-India fellows in summer 2015. 
\end{IEEEbiographynophoto}

\begin{IEEEbiographynophoto}{Ramtin Pedarsani}
is an Assistant Professor in ECE Department at the University of California, Santa Barbara. He received the B.Sc. degree in electrical engineering from the University of Tehran, Tehran, Iran, in 2009, the M.Sc. degree in communication systems from the Swiss Federal Institute of Technology (EPFL), Lausanne, Switzerland, in 2011, and his Ph.D. from the University of California, Berkeley, in 2015. His research interests include machine learning, information and coding theory, networks, and transportation systems. Ramtin is a recipient of the IEEE international conference on communications (ICC) best paper award in 2014.
\end{IEEEbiographynophoto}

\begin{IEEEbiographynophoto}{A. Salman Avestimehr}
is a Professor and director of the Information Theory and Machine Learning (vITAL) research lab at the Electrical and Computer Engineering Department of University of Southern California. He received his Ph.D. in 2008 and M.S. degree in 2005 in Electrical Engineering and Computer Science, both from the University of California, Berkeley. Prior to that, he obtained his B.S. in Electrical Engineering from Sharif University of Technology in 2003. His research interests include information theory, coding theory, and large-scale distributed computing and machine learning.

Dr. Avestimehr has received a number of awards for his research, including the James L. Massey Research \& Teaching Award from IEEE Information Theory Society, an Information Theory Society and Communication Society Joint Paper Award, a Presidential Early Career Award for Scientists and Engineers (PECASE) from the White House, a Young Investigator Program (YIP) award from the U. S. Air Force Office of Scientific Research, a National Science Foundation CAREER award, the David J. Sakrison Memorial Prize, and several Best Paper Awards at Conferences. He is a Fellow of IEEE. He has been an Associate Editor for IEEE Transactions on Information Theory. He is currently a general Co-Chair of the 2020 International Symposium on Information Theory (ISIT).
\end{IEEEbiographynophoto}

%% file: 5-appA.tex
\begin{algorithm}[h!]
\caption{Computation Allocation}  \label{alg:CompAlloc}
\hspace*{\algorithmicindent} \textbf{Input:} dataset $\cD$, $n$ workers, straggler toleration $s$, computation matrix $\bB = [\bbb_1;\cdots;\bbb_n] \in \mathbb{R}^{n \times k}$ \\
\hspace*{\algorithmicindent} \textbf{Output:} data set allocation $\{\cD_{(1)},\cdots,\cD_{(n)}\}$ for $n$ workers
\begin{algorithmic}[1]
\Procedure{CompAlloc($\cD, \bB$)}{}
\State uniformly partition $\cD = \cup_{\kappa=1}^{k} \cD_\kappa$
\For{worker $i \gets 1$ to $n$}
\State $\cD_{(i)} \gets \cup_{\kappa=1}^{k} b_{i\kappa} \cD_\kappa$  \Comment{$\cD_{(i)}$ is assigned to worker $W_i$}
\EndFor
\EndProcedure
\end{algorithmic}
\end{algorithm}

%% file: 5-appB.tex
\begin{algorithm}[h!]
\caption{CodedReduce}  \label{alg:CR}
\hspace*{\algorithmicindent} \textbf{Input:} dataset $\cD$, $(n,L)$--regular tree $T$, straggler toleration $s$ (per parent), model $\theta^{(t)}$ \\
\hspace*{\algorithmicindent} \textbf{Output:} gradient $\bg_{\cD} = \sum_{\bx \in \cD} \nabla \ell (\theta^{(t)}; \bx )$  aggregated at the master
\begin{algorithmic}[1]
\Procedure{CR.Allocate}{}
\State \textsf{GC} generates $\bB$ specified by $n,s$
\For{$l \gets 1$ to $L$}
\For{$i \gets 1$ to $n^{l-1}$}
\State $ \{ \cD^{T(l,n(i-1)+1)},\cdots,\cD^{T(l,ni)} \} = \textsc{CompAlloc}(\cD_{T(l-1,i)}, \bB)$
\EndFor
\For{$i \gets 1$ to $n^{l}$}
\State pick $r_{\textsf{CR}} \cdot d$ data points of $\cD^{T(l,i)}$ as $\cD(l,i)$
\State $ \cD_{T(l,i)} \gets \cD^{T(l,i)} \setminus \cD(l,i)$
\EndFor
\EndFor
\EndProcedure
\Procedure{CR.Execute}{}
\State \textsf{GC} generates $\bA$ from $\bB$
\State all the workers compute their local partial gradients $\bg_{\cD(l,i)}$
\For{$l \gets L-1$ to $1$}
\For{$i \gets 1$ to $n^{l}$}
\State worker nodes $(l,i)$:
\State \quad receives $[\bbm_{(l+1,n(i-1)+1)}; \cdots; \bbm_{(l+1,ni)}]$ from its children
\State \quad uploads $ \bbm_{(l,i)} = \ba_{f(l,i)}[ \bbm_{(l+1,n(i-1)+1)}; \cdots; \bbm_{(l+1,ni)}] + \bg_{\cD(l,i)}$ to its parent
\EndFor
\EndFor
\State master node:
\State \quad receives $[ \bbm_{(1,1)}; \cdots; \bbm_{(l,n)}]$ from its children
\State \quad recovers $\bg = \ba_{f(0,1)} [ \bbm_{(1,1)}; \cdots; \bbm_{(1,n)}]$
\EndProcedure
\end{algorithmic}
\end{algorithm}

%% file: 5-appC.tex
\textbf{Achievability:} According to the data allocation described in Algorithm \ref{alg:CR}, to be robust to any $s$ straggling children of the master, the data set $\cD$ is redundantly assigned to sub-trees $T(1,1),\cdots,T(1,n)$ such that each data point is placed in $s+1$ sub-trees, which yields
\begin{equation} \label{eq:rCR_proof1}
    |\cD^{T(1,i)}| = \left( \frac{s+1}{n} \right) d, \quad \forall  i \in [n].
\end{equation}
Then, nodes in layer $l=1$ pick $r_{\textsf{CR}}d$ data points as their corresponding data sets and similarly distribute the remaining among their children which together with (\ref{eq:rCR_proof1}) yields
\begin{align}
    |\cD^{T(2,i)}| &= \left( \frac{s+1}{n} \right) \left( \left( \frac{s+1}{n} \right)d -  r_{\textsf{CR}} d \right) \nonumber\\
    & = \left( \frac{s+1}{n} \right) \left( \left( \frac{s+1}{n} \right) - r_{\textsf{CR}} \right) d, \quad \forall  i \in [n^2]. \nonumber
\end{align}
By the same argument for each layer, we have
\begin{align} \label{eq:rCR_proof2}
    |\cD^{T(L,i)}| 
    &=
    \left( \frac{s+1}{n} \right) \Bigg( \left( \frac{s+1}{n} \right)^{L-1} - \left( \frac{s+1}{n} \right)^{L-2} r_{\textsf{CR}} \nonumber\\
    & \quad
    - \cdots - \left( \frac{s+1}{n} \right) r_{\textsf{CR}} - r_{\textsf{CR}} \Bigg)d, \quad \forall  i \in [n^L].
\end{align}
Putting (\ref{eq:rCR_proof2}) together with $|\cD^{T(L,i)}| = r_{\textsf{CR}} d$ yields
\begin{equation}
     r_{\textsf{CR}} = \frac{1}{ \left( \frac{n}{s+1} \right) + \cdots + \left( \frac{n}{s+1} \right)^L }. \nonumber
\end{equation}

\textbf{Optimality:} In an $\alpha$--resilient scheme, the master node is able to recover from any $s = \alpha n$ straggling sub-trees $T(1,1),\cdots,T(1,n)$. Therefore, each data point has to be placed in at least $s+1$ of such sub-trees, which yields
\begin{equation}\label{eq:rlower1}
    |\cD^{T(1,1)}| + \cdots + |\cD^{T(1,n)}| \geq (s+1)d,
\end{equation}
where the equality is achieved only if each data point is assigned to only $s+1$ sub-trees. Hence, we can assume the optimal scheme satisfies (\ref{eq:rlower1}) with equality. Moving to the second layer, the following claim bounds the required redundancy assigned to sub-trees $T(2,1),\cdots,T(2,n)$. Similar claim holds for any other group of the siblings in this layer.
\begin{claim} \label{claim1}
The following inequality holds:
$$|\cD^{T(2,1)}| + \cdots + |\cD^{T(2,n)}| \geq (s+1) \left( |\cD^{T(1,1)}| - rd \right).$$
\end{claim}
\begin{proof}[Proof of Claim \ref{claim1}]
First, note that $|\cD^{T(1,1)} \setminus \cD{(1,1)}| \geq |\cD^{T(1,1)}| - rd$. If the claim does not hold, then there exists data point $\mathbf{x} \in \cD^{T(1,1)} \setminus \cD{(1,1)}$ such that $\mathbf{x}$ is placed in at most $s$ sub-trees rooting in the node $(1,1)$, e.g. $T(2,1),\cdots,T(2,s)$. Note that besides sub-tree $T(1,1)$, $\mathbf{x}$ is placed in only $s$ more sub-trees, e.g. $T(1,2),\cdots,T(1,s+1)$.  Now consider a straggling pattern where $T(1,2),\cdots,T(1,s+1)$ and $T(2,1),\cdots,T(2,s)$ fail to return their results. Therefore, $\mathbf{x}$ is missed at the master and fails the aggregation recovery.
\end{proof}

By the same logic used in the above proof, Claim \ref{claim1} holds for any parent node and its children, i.e. for any layer $l \in [L]$ and $i \in [n^{l-1}]$,
\begin{equation} \label{eg:generalbound}
    |\cD^{T(l,n(i-1)+1)}| + \cdots + |\cD^{T(l,ni)}| \geq (s+1) \left( |\cD^{T(l-1,i)}| - rd \right).
\end{equation}
Specifically applying (\ref{eg:generalbound}) to layer $L$ and noting that $|\cD^{T(L,i)}|=|\cD{(L,i)}|=rd$ for any $i$, we conclude that 
$$rd \left( \left( \frac{n}{s+1} \right)+1 \right) \geq |\cD^{T(L-1,1)}|.$$ 
We can then use the above inequality and furthermore write (\ref{eg:generalbound}) for layer $L-1$ which results in 
\begin{equation} 
    rd \left( \left( \frac{n}{s+1} \right)^2 + \left( \frac{n}{s+1} \right) +1 \right) \geq |\cD^{T(L-2,1)}|. \nonumber
\end{equation}
By deriving the above inequality recursively up to the master node, we get
\begin{equation} 
    rd \left( \left( \frac{n}{s+1} \right)^{L-1} + \cdots + \left( \frac{n}{s+1} \right) +1 \right) \geq \frac{s+1}{n}d, \nonumber
\end{equation}
which concludes the optimality in Theorem \ref{thm:CRoptimality}.

%% file: 5-appD.tex
Let us begin with the lower bound 
\begin{align}
    \Expc \left[T_{\textsf{CR}} \right] &\geq \frac{r_{\textsf{CR}} d}{\mu} \log \left( \frac{1}{\alpha} \right) + a r_{\textsf{CR}} d \nonumber\\ 
    & + \left( n(1-\alpha)-o(n)+L-1 \right) \left( (1-o(1) \right) t_c + o(1).\nonumber
\end{align}
Consider the group of siblings\footnote{A group of siblings refers to $n$ nodes with the same parent.} placed in layer $L$ whose result reaches their parent nodes first. Let $\widehat{T}$ denote the time at which the parent of such group is able to recover the partial gradient from its fastest children's computations, i.e. fastest $n-s$ of them. We also denote by $T_1,\cdots,T_n$ the partial gradient computation times for the siblings. According to the random computation time model described in the paper and the computation load of \textsf{CR}, each $T_i$ is shifted exponential with the shift parameter $a d_i = a r_{\textsf{CR}} d $ and the rate parameter $\frac{\mu}{d_i} = \frac{\mu}{r_{\textsf{CR}} d}$. Since \textsf{CR} is robust to any $s$ stragglers per parent, the partial gradient computation time for any group of siblings is $T_{(n-s)}$, i.e. the $(n-s)$'th order statistics of $\{ T_1,\cdots,T_n \}$. In \cite{reisizadeh2017latency}, authors consider coded computation scenarios in a master-worker topology where the master only needs to wait for results of the first $\alpha$ fraction of the workers. However, as in the scenario here, the limited bandwidth at the master only allows for one transmission at the time. From the latency analysis in \cite{reisizadeh2017latency}, we have the following.
\begin{lemma}[Theorem 2, \cite{reisizadeh2017latency}] \label{lemma:regimeii}
With probability $1-o(1)$, we have 
\begin{align} \label{eq:regimeii}
    \widehat{T} \geq T_{(n-s)} + \left( n \left( 1-\alpha \right) - o(n) \right) t_c. 
\end{align}
\end{lemma}
Now, conditioned on the event in (\ref{eq:regimeii}) we can write
\begin{align}
    \Expc \left[T_{\textsf{CR}} \right] 
    & \geq  \left( \Expc \left[T_{(n-s)} \right] + \left( n \left( 1-\alpha \right) - o(n) \right) t_c \right) \left( 1 - o(1) \right) \nonumber\\
    & \quad + \left( \Expc \left[T_{(n-s)} \right] + L t_c \right)  o(1)  \nonumber \\
    & \geq \Expc \left[T_{(n-s)} \right]  \nonumber \\
    & \quad + \left( n(1-\alpha)-o(n)+L-1 \right) \left( 1-o(1) \right) t_c \nonumber \\
    & \overset{(a)}\geq \frac{r_{\textsf{CR}} d}{\mu} \log \left( \frac{1}{\alpha} \right) + a r_{\textsf{CR}} d  \nonumber \\
    & \quad+ \left( n(1-\alpha)-o(n)+L-1 \right) \left( 1-o(1) \right) t_c + o(1), \nonumber
\end{align}
where inequality $(a)$ uses the fact that $\Expc \left[T_{(n-s)} \right] = \frac{r_{\textsf{CR}} d}{\mu} \left( H_n - H_s \right) + a r_{\textsf{CR}} d $ and  $ \log(i) < H_i = 1 + \frac{1}{2} + \cdots + \frac{1}{i} < \log(i+1)$ for any positive integer $i$.

To derive upper bound on $\Expc [T_{\textsf{CR}} ]$, that is
\begin{align}
    \Expc \left[T_{\textsf{CR}} \right] \leq 
    \frac{r_{\textsf{CR}} d}{\mu} \log \left( \frac{1}{\alpha} \right) + a r_{\textsf{CR}} d  +  n \left( 1-o(1) \right) L t_c + o(1),\nonumber
\end{align}
we prove the following concentration inequality on the computation time for any group of siblings.

\begin{lemma} \label{lemma:orderstat}
Let $T_1,\cdots,T_n$ denote i.i.d. exponential random variables with constant rate $\lambda = \Theta(1)$. For  $\varepsilon = \Theta \left( \frac{1}{n^{1/4}} \right)$ and constant $\alpha = \frac{s}{n}$, we have the following concentration bound for the order statistics $T_{(n-s)}$:
\begin{align}
    \Prob \left[T_{(n-s)} - \Expc \left[ T_{(n-s)}\right] \geq \varepsilon \right] \leq e^{- \Theta \left( \sqrt{n} \right)}.
\end{align}
\end{lemma}
\begin{proof}[Proof of Lemma \ref{lemma:orderstat}]
Given i.i.d. exponentials $T_1,\cdots,T_n \sim \exp(\lambda)$, we can write the successive differences of order statistics as independent exponentials. That is, we have
\begin{align}
    T_{(1)} &= E_1 \sim \exp \left( \frac{\lambda}{n} \right), \nonumber \\
    T_{(2)} - T_{(1)} &= E_2 \sim \exp \left( \frac{\lambda}{n-1} \right), \nonumber \\
    & \vdots  \nonumber \\
    T_{(n-s)} - T_{(n-s-1)} &= E_{n-s} \sim \exp \left( \frac{\lambda}{s+1} \right), \nonumber \\
    & \vdots \nonumber \\
    T_{(n)} - T_{(n-1)} &= E_n \sim \exp \left( \lambda \right),\nonumber 
\end{align}
where $E_i$'s are independent. Thus, $T_{(n-s)} = \sum_{i=1}^{n-s} E_i$. We have the following for independent exponentials $E_i$'s and $\lambda = \Theta(1)$:
\begin{align}
    \Expc \left[ |E_i|^k \right] &= \Expc \left[ E_i^k \right] \nonumber \\
    &= \left( \frac{\lambda}{n-i+1} \right)^k k! \nonumber \\
    &= \frac{1}{2} \Expc \left[ E_i^2 \right] \left( \frac{\lambda}{n-i+1} \right)^{k-2} k! \nonumber \\
    & \leq \frac{1}{2} \Expc \left[ E_i^2 \right] B^{k-2} k!, \nonumber 
\end{align}
for $B = \frac{\lambda}{s} = \frac{\lambda}{\alpha n} = \Theta \left( \frac{1}{n} \right)$. Moreover, 
\begin{align}
    \sum_{i=1}^{n-s} \Expc \left[ E_i^2 \right] &= 2 \lambda^2 \left( \frac{1}{n^2} + \cdots + \frac{1}{(s+1)^2} \right) \nonumber \\
    & \leq 2 \lambda^2 \cdot \frac{n-s}{s^2} \nonumber \\
    &= \frac{2 \lambda^2 (1-\alpha)}{\alpha^2} \cdot \frac{1}{n} \nonumber \\
    & = \Theta \left( \frac{1}{n} \right). \nonumber 
\end{align}
According to Bersterin's Lemma (See Lemma \ref{lemma:bernstein}), for $\varepsilon = \Theta \left( \frac{1}{n^{1/4}} \right)$ we have
\begin{align}
    &\Prob \left[T_{(n-s)} - \Expc \left[ T_{(n-s)}\right] \geq \varepsilon \right]  \nonumber \\
    & \quad \leq 
    \exp \left( - \frac{\varepsilon^2}{2 \left( \sum_{i=1}^{n-s} \Expc \left[ E_i^2 \right] + \varepsilon B \right)} \right) \nonumber \\
    &\quad \leq \exp \left( - \frac{\varepsilon^2}{2 \left( \Theta \left( \frac{1}{n} \right) + \varepsilon \Theta \left( \frac{1}{n} \right) \right)} \right) \nonumber \\
    & \quad = e^{ - \Theta \left( \sqrt{n} \right) }.\nonumber 
\end{align}
\end{proof}
As described in Section \ref{subsec:CRdescr}, in the proposed \textsf{CR} scheme all the worker nodes start their assigned partial gradient computations simultaneously; each parent waits for enough number of children to receive their results; combines with its partial computation and sends the result up to its parent. To upper bound the total aggregation time $T_{\textsf{CR}}$, one can separate all the local computations from the communications. Let $T_{\text{comp}}$ denote the time at which enough number of workers have executed their local gradient computations and no more local computation is needed for the final gradient recovery. Moreover, we assume that all the communications from children to parent are pipe-lined. Hence, we have $\Expc \left[ T_{\textsf{CR}} \right] \leq \Expc \left[ T_{\text{comp}} \right] + L (n-s) t_c$. To bound the computation time $T_{\text{comp}}$, consider the following event which keeps the local computation times for \emph{all} the $N/n$ groups of siblings concentrated below their average deviated by $\varepsilon = \Theta \left( \frac{1}{n^{1/4}} \right)$:
\begin{align}
    \mathcal{E}_1 \coloneqq \left\{ T^{gr}_{(n-s)} \leq \Expc \left[ T^{gr}_{(n-s)}\right] + \varepsilon \text{ for all the $N/n$ groups $gr$} \right\},\nonumber 
\end{align}
where a group $gr$ is a collection of $n$ children with the same parent, i.e. there are $N/n$ groups in the $(n,L)$--regular tree. For a group $gr$, $\{T^{gr}_{1}, \cdots, T^{gr}_{n}\}$ denote the random run-times of the nodes in the group and $T^{gr}_{(n-s)}$ represents its $(n-s)$'th order statistics.
Clearly,
\begin{align} \label{eq:bound0}
    \Expc \left[ T_{\text{comp}} | \mathcal{E}_1 \right] \leq \Expc \left[ T_{(n-s)}\right] + o(1).
\end{align}
Now let $\widetilde{T}$ denote the computation time corresponding to the slowest group of siblings, i.e.
\begin{align}
    \widetilde{T} \coloneqq \max_{\text{ over all $N/n$ groups $gr$}} T^{gr}_{(n-s)}.\nonumber 
\end{align}
Consider the following event:
\begin{align}
    \mathcal{E}_2 \coloneqq \left\{ \widetilde{T} > \Theta( \log n) \right\}.\nonumber 
\end{align}
We can write
\begin{align} \label{eq:bound1}
    \Expc \left[ T_{\text{comp}} | \mathcal{E}^c_1 \cap \mathcal{E}_2^c \right] \leq \Theta( \log n),
\end{align}
and
\begin{align} \label{eq:bound2}
    \Expc \left[ T_{\text{comp}} | \mathcal{E}^c_1 \cap \mathcal{E}_2 \right] & \Prob \left[\mathcal{E}_2 \right] \nonumber \\
    & \leq \Expc \left[ \widetilde{T} | \mathcal{E}^c_1 \cap \mathcal{E}_2 \right] \Prob \left[\mathcal{E}_2 \right] \nonumber \\
    & = \Expc \left[ \widetilde{T} | \widetilde{T} \leq \Theta( \log n) \right] \Prob \left[\widetilde{T} \leq \Theta( \log n) \right] \nonumber \\
    & \leq \Expc \left[ \widetilde{T} \right] \nonumber \\
    & \leq \Expc \left[ T_{\text{max}} \right] \nonumber \\
    & = \frac{r_{\textsf{CR}} d}{\mu} H_N + a r_{\textsf{CR}} d \nonumber \\
    & = \Theta \left( \log N \right) \nonumber \\
    & = L \Theta \left( \log n \right).
\end{align}
In the above derivation, $ T_{\text{max}}$ denotes the largest computation time over all the $N$ nodes. Putting (\ref{eq:bound1}) and (\ref{eq:bound2}) together, we can write 
\begin{align} \label{eq:bound4}
    \Expc \left[ T_{\text{comp}} | \mathcal{E}^c_1 \right]  &= \Expc \left[ T_{\text{comp}} | \mathcal{E}^c_1 \cap \mathcal{E}_2 \right] \Prob \left[\mathcal{E}_2 \right] \nonumber \\
    &\quad + \Expc \left[ T_{\text{comp}} | \mathcal{E}^c_1 \cap \mathcal{E}_2^c \right] \Prob \left[\mathcal{E}_2^c \right] \nonumber \\
    &\leq \Theta \left( \log n \right).
\end{align}
Moreover, using union bound on the $N/n$ groups of workers, we derive the following inequality.
\begin{align} \label{eq:bound3}
    \Prob \left[\mathcal{E}_1^c \right] & \leq \frac{N}{n} \Prob \left[T_{(n-s)} \geq \Expc \left[ T_{(n-s)}\right] + \varepsilon \right] \nonumber \\
    & \leq \Theta \left( n^{L-1} \right) e^{ - \Theta \left( \sqrt{n} \right) }.
\end{align}
Putting (\ref{eq:bound0}), (\ref{eq:bound4}) and (\ref{eq:bound3}) together, we have
\begin{align}
    \Expc \left[ T_{\text{comp}} \right] 
    &= \Expc \left[ T_{\text{comp}} | \mathcal{E}_1 \right] \Prob \left[\mathcal{E}_1 \right]
    + \Expc \left[ T_{\text{comp}} | \mathcal{E}_1^c \right] \Prob \left[\mathcal{E}_1^c \right] \nonumber \\
    & \leq \Expc \left[ T_{(n-s)}\right] + \varepsilon + \Theta \left( \log n \right) \Theta \left( n^{L-1} \right) e^{ - \Theta \left( \sqrt{n} \right) } \nonumber \\
    & = \Expc \left[ T_{(n-s)}\right] + o(1) \nonumber \\
    & = \frac{r_{\textsf{CR}} d}{\mu} \left( H_n - H_s \right) + a r_{\textsf{CR}} d + o(1).\nonumber 
\end{align}
Therefore,
\begin{align}
    \Expc \left[T_{\textsf{CR}} \right] &\leq \Expc \left[ T_{\text{comp}} \right] + L  n(1-\alpha) t_c \nonumber \\
    &= \frac{r_{\textsf{CR}} d}{\mu} \left( H_n - H_s \right) + a r_{\textsf{CR}} d + L n(1-\alpha) t_c + o(1) \nonumber \\
    & \leq 
    \frac{r_{\textsf{CR}} d}{\mu} \log \left( \frac{1}{\alpha} \right) + a r_{\textsf{CR}} d +  n \left( 1-o(1) \right) L t_c + o(1),\nonumber 
\end{align}
which completes the proof.

\begin{lemma}[Bernstein's Inequality] \label{lemma:bernstein}
Suppose $E_1,\cdots,E_m$ are independent random variables such that
\begin{align}
    \Expc \left[ |E_i|^k \right] \leq \frac{1}{2} \Expc \left[ E_i^2 \right] B^{k-2} k!,\nonumber 
\end{align}
for some $B > 0$ and every $i=1,\cdots,m$, $k \geq 2$. Then, for $\varepsilon > 0$,
\begin{align}
    & \Prob \left[ \sum_{i=1}^{m} E_i - \sum_{i=1}^{m} \Expc \left[ E_i \right] \geq \varepsilon \right] \nonumber \\ 
    & \quad \leq \exp \left( - \frac{\varepsilon^2}{2 \left( \sum_{i=1}^{m} \Expc \left[ E_i^2 \right] + \varepsilon B \right)} \right).\nonumber 
\end{align}
\end{lemma}